\documentclass[11pt, letterpaper]{article}
\usepackage[margin=1in]{geometry}
\usepackage{ifpdf}
\ifpdf
    \usepackage[pdftex]{graphicx}
    \usepackage[update]{epstopdf}
\else
	\usepackage{graphicx}
\fi
\usepackage{wrapfig,bbm}
\usepackage{enumitem,color}
\usepackage{amsthm,multirow}
\newtheorem{theorem}{Theorem}
\newtheorem{lemma}{Lemma}
\newtheorem{corollary}{Corollary}

\usepackage[pdfencoding=auto]{hyperref}
\usepackage{amssymb}

\theoremstyle{definition}
\newtheorem{definition}{Definition}[section]
\newtheorem{proposition}{Proposition}
\usepackage{multirow}
\usepackage{algorithm}
\usepackage{algorithmic}
\usepackage{amsmath}
\usepackage{centernot}

\newcommand{\CI}{\raisebox{0.05em}{\rotatebox[origin=c]{90}{$\models$}}}
\newcommand{\nCI}{\centernot{\CI}}

\DeclareMathOperator\arctanh{arctanh}

\newtheorem{assumption}{Assumption}

\begin{document}

\title{A Fast PC Algorithm with Reversed-order Pruning and A Parallelization Strategy}
\author{Kai Zhang, Chao Tian, Todd Johnson, Kun Zhang, Xiaoqian Jiang\thanks{Kai Zhang, Todd Johnson and Xiaoqian Jiang are with the School of Biomedical Informatics, UT Health Science Center at Houston, Houston, TX 77030, USA (email: \{Kai.Zhang.1, Todd.R.Johnson, Xiaoqian.Jiang\}@uth.tmc.edu), Chao Tian is with the Department of Electrical and Computer Engineering, Texas A\&M University, College Station, TX 77843, USA (email: Chao.Tian@tamu.edu), and Kun Zhang is with the Department of Philosophy, Carnegie Mellon University, Pittsburgh, PA 15213, USA (email: Kunz1@cmu.edu).}}
\date{}
\maketitle

\begin{abstract}
The PC algorithm is the state-of-the-art algorithm for causal structure discovery on observational data. It can be computationally expensive in the worst case due to the conditional independence tests are performed in an exhaustive-searching manner. This makes the algorithm computationally intractable when the task contains several hundred or thousand nodes, particularly when the true underlying causal graph is dense. We propose a critical observation that the conditional set rendering two nodes independent is non-unique, and including certain redundant nodes do not sacrifice result accuracy. Based on this finding, the innovations of our work are two-folds. First, we innovate on a reserve order linkage pruning PC algorithm which significantly increases the algorithm's efficiency. Second, we propose a parallel computing strategy for statistical independence tests by leveraging tensor computation, which brings further speedup. We also prove the proposed algorithm does not induce statistical power loss under mild graph and data dimensionality assumptions. Experimental results show that the single-threaded version of the proposed algorithm can achieve a 6-fold speedup compared to the PC algorithm on a dense 95-node graph, and the parallel version can make a 825-fold speed-up. We also provide proof that the proposed algorithm is consistent under the same set of conditions with conventional PC algorithm. 
\end{abstract}

\section{Introduction}
Causal discovery has become an appealing research direction, by its ability of extracting meaningful cause-effect information from the otherwise unwieldy mass of data. It has been successfully applied to various disciplines of sciences, such as statistics \cite{pearl2009causal}, public policy \cite{zajonc2012essays}, economics \cite{varian2016causal}, biology \cite{shipley2016cause} and health care analytics \cite{glass2013causal}, and it has revolutionized the way we utilize clinical data.

The PC algorithm is one of the most successful and widely-used constraint-based algorithms that have asymptotic correctness guarantee \cite{spirtes1991algorithm}. A comprehensive study on the accuracy of various causal discovery algorithms \cite{singh2017comparative} demonstrated that on large datasets, the PC algorithm outperformed other algorithms, including GES, FCI, FCI+, MMHC, and the Active Learning Method \cite{tong2001active} on four standard performance metrics, the F-score, area under the receiver operating characteristic curve (AUC), structural Hamming distance (SHD) and structural intervention distance (SID). The PC algorithm including its variants have also been implemented by several open source libraries and packages, such as Tetrad Toolbox \cite{scheines1998tetrad}, Causal Discovery Toolbox \cite{kalainathan2019causal}, R package pcalg \cite{kalisch2012causal} and bnlearn \cite{scutari2009learning}, etc. 

The PC algorithm does not scale well on high-dimensional settings due to the high time complexity caused by the heuristic searching in the algorithm. To this end, Le at al. proposed a parallel PC algorithm by grouping CI tests and distributing them over different processes, and demonstrated their algorithm could finish within 6 hours on an 8-core CPU on some test datasets on which PC algorithm could not finish within 24 hours \cite{le2016fast}. Madsen et al. proposed two parallelization methods: a horizontal parallel approach that showed speed-up on both shared memory system and a cluster system when using processes, and a Balanced Incomplete Block (BIB)-design-based approach for marginal independence testing, which shows speed-up on shared memory systems when using threads \cite{madsen2017parallel}. Zare et al. introduced two parallel PC algorithms, cuPC-E, and cuPC-S, achieving 500-fold and 1300-fold speed-up compared to the serial implementation, respectively \cite{zare2018cupc}. The former adopts the idea of parallelizing both the tests for multiple edges and multiple tests for a given edge, however, this algorithm incurs additional overhead and unnecessary tests whereas the latter employs a local sharing strategy to avoid such overhead. 

However, PC algorithm's intrinsic searching strategy entails the time complexity of factorial order, and the parallelization techniques only bring very limited speed gain. This promotes us to pursue algorithm-level improvements rather than parallization tricks. In this study, we propose an improved PC algorithm that significantly improves the speed of the PC algorithm without scarifying the result accuracy. The main contributions are as follows.

\begin{itemize}
\item \textbf{Reverse-order pruning PC algorithm}. We propose a reverse-order pruning PC algorithm that shows significant improvement in the algorithm speed. We prove the new algorithm is asymptotically consistent on high dimensional datasets under the same set of conditions of PC algorithm without enforcing any further assumptions. 

\item \textbf{Parallelization strategy and GPU-acceleration}. We propose a novel parallelization strategy for computing partial correlation coefficients by utilizing tensor computations. Such parallelization can be accelerated with GPUs. 
\end{itemize}

The rest of the paper is organized as follows. Section 2 provides some backgstage on causal discovery. Section 3 presents our main algorithm. In section 4, we provide the parallelization strategy and the GPU acceleration. Section 5 includes a detailed theoretical analysis of the consistency of the algorithm's result. Section 7 concludes our work. Our code is available on Github \footnote{https://github.com/anotherkaizhang/fastPC}.

\section{Preliminaries}
\subsection{Bayesian Network}
\textit{Bayesian network} is a probabilistic graphical model using directed acyclic graphs (DAG) to represent the joint probability distribution of a set of random variables, which provides a comprehensive and compact way to represent the interactions among random variables. A Bayesian network is referred to as a \textit{causal Bayesian network}, or a \textit{causal graph}, \textit{causal network} \cite{heckerman1995learning} when the directed edges are logically interpreted as cause-and-effect relation among random variables,

A graph is mathematically defined as $G=(\mathcal{V}, \mathcal{E})$, where $\mathcal{V}$ is the index set of a non-empty set of nodes (or vertices) $X_{\mathcal{V}}=\{X_i: i \in \mathcal{V}\}$, and $\mathcal{E}$ is a (possibly empty) set of undirected or directed edges. 
In this work, we only focus on the acyclic graphs, and the graph $G$ can be completely directed (DAG), completely undirected, or partially directed (PDAG). For a directed edge $X_i \rightarrow X_j$, $X_i$ is called a \textit{parent} of $X_j$ and $X_j$ is called a \textit{child} of $X_i$. The \textit{ancestors} of $X_i$ is defined as its parents, and parents of parents, and so forth (continues recursively). Similarly, its \textit{descendents} is defined as its children, children of children, and so forth. If there is an undirected or directed edge between two nodes $X_i, X_j$, then the two nodes $X_i, X_j$ are called \textit{adjacent}. We use $adj(G, i)$ to denote the set of nodes adjacent to the node $i$ in graph $G$, and $adj(G,i,j)$ to denote the set of nodes adjacent to either node $i$ or $j$ in graph $G$, except for $i,j$ themselves. A \textit{path} is a sequence of distinct adjacent nodes. A \textit{directed path} between two variables $X_i, X_j$ is defined by identifying a sequence of nodes $\{X_{k_1}, \ldots, X_{k_m}\}$ such that there exists a consecutive sequence of directed edges $X_i \rightarrow X_{k_1}$, $X_{k_1} \rightarrow X_{k_2}$, $\ldots$, $X_{k_{m-1}} \rightarrow X_{k_m}$ and $X_{k_m} \rightarrow X_j$. The $path(G, i, j)$ is used to denote the set of nodes on the undirected paths between $i$ and $j$ in undirected graph $G$. A \textit{directed cycle} is defined as a directed path from a node to itself. The \textit{skeleton} of a DAG is the graph after removing the direction of all edges. A \textit{completed partially directed acyclic graph (CPDAG)} is a DAG with some edges that have undetermined direction. A \textit{$v$-structure} (or \textit{unshielded collider}) in graph $G$ is a triple $(X_i, X_k, X_j)$ of structure $X_i \rightarrow X_k \leftarrow X_j$ and the nodes $X_i, X_j$ are not adjacent. The PC algorithm uses \textit{$d$-separation} on the causal graph to infer conditional independence relationship embedded in the underlying probabilistic model. 

\begin{definition}[$d$-separation \cite{lauritzen1996graphical}]
A path is $d$-separated by a set of nodes $\mathcal{Z}$ if and only if \\
(i) the path contains a chain $X_i \rightarrow X_m \rightarrow X_j$ or a fork $X_i \leftarrow X_m \rightarrow X_j$ such that the middle node is in $\mathcal{Z}$, or \\
(ii) the path contains a collider $X_i \rightarrow X_m \leftarrow X_j$ such that the middle node $X_m$ and its descendants are not in $\mathcal{Z}$. \\
$X_i$ and $X_j$ are said to be $d$-separated by a set of nodes $\mathcal{Z}$ if and only if $\mathcal{Z}$ blocks every path between $X_i$ and $X_j$, denoted as $X_i \CI X_j \lvert \mathcal{Z}$.
\end{definition}

In the following, we use script letters to denote a set of variables or nodes, such as $\mathcal{X}$, and sometimes use $X_{\mathcal{V}}$ for the same purpose, when we want to emphasize the random variables (or nodes) that are indexed by the index set $\mathcal{V}$. We use $X_i$ to interchangeably denote both a random variable and a node in the DAG.

The causal Markov and faithfulness conditions, together with the independent and identically distributed (i.i.d) sampling, no latent variables, and no selection bias, constitute a sufficient set of assumptions for the PC algorithm to converge to the correct Markov equivalence class in the limit of infinite sample size \cite{spirtes2000causation}. 

\begin{assumption}[Markov Condition] \cite{spirtes2000causation}]\label{as:causalmarkov} 
A directed acyclic graph (DAG) $G(\mathcal{V}, \mathcal{E})$ and a probability distribution $P(X_{\mathcal{V}})$ satisfy the Markov condition if and only if for every variable $X_i, i \in \mathcal{V}$, $X_i$ is independent of $X_{\mathcal{V}} \backslash (descendants(X_i) \cup parents(X_i))$ given $parents(X_i)$. 
\end{assumption}

The Markov condition implies that the joint density function of $X_{\mathcal{V}}$ can be represented by 

\begin{equation*}
f(X_{\mathcal{V}}) = \prod_{X \in X_{\mathcal{V}}} f(X \; \lvert \; parents(X)). 
\end{equation*}

The Markov condition ensures every independence relation acquired by implying $d$-separation on the causal graph holds in the joint probabilistic distribution. The \textit{Causal Markov Condition} states that a node is independent of all other nodes except its direct cause and effects, conditioning on all of its direct cause variables \cite{hausman1999independence}. The causal Markov condition and Markov condition are equivalent if the Bayesian network is a causal graph. 

\begin{assumption}[Faithfulness Condition \cite{judea2000causality,spirtes2000causation}]\label{as:faithfulness} 
The conditional independences in the causal graph are exactly those in the probability distribution.
\end{assumption}

The faithfulness condition assures that there are no additional independencies in the joint probabilistic model which cannot be obtained by applying $d$-separation on the causal graph. This assumption is necessary because, among all joint probabilistic distributions entailed by the causal graph, a deliberate tuning of the parameters of the functional model could produce extra independence relations that cannot be successfully captured by the DAG \cite{cartwright1984laws}. The faithfulness assumption rules out such very unlikely to happen but still existing events. 

If both the causal Markov condition and the faithfulness condition hold, then there is a bijection relationship between the conditional independence implied by $d$-separation on DAG and those embraced by the underlying probabilistic model.

\begin{center}
\textit{Node $X$ and $Y$ are $d$-separated by the set of nodes $\mathcal{Z}$ in the DAG $\Leftrightarrow$ Variable $X$ and $Y$ are independent given the set of variables $\mathcal{Z}$}.
\end{center}

\subsection{The PC and PC-stable Algorithm}
Constraint-based causal discovery algorithms such as PC and PC-stable in general produce a CPDAG \cite{chickering2002optimal, spirtes2000causation}, representing a Markov equivalence class \cite{meek2013causal} of DAGs.

\begin{theorem}[Markov Equivalence Class \cite{verma1991equivalence}]
Two directed acyclic graphs $G(\mathcal{V}, \mathcal{E})$ and $G'(\mathcal{V}, \mathcal{E})$ are Markov equivalent if and only if they have the same adjacencies and same unshielded colliders.
\end{theorem}

In other words, the Markov equivalence class denotes a set of DAGs that encode the same set of conditional independencies \cite{he2015counting}. It explains why constraint-based algorithms can differentiate distinct Markov equivalence classes but are incapable of distinguishing CPDAGs within the same equivalence class.

In the following, we briefly discuss the \textit{Oracle Version} of the PC and PC-stable algorithm and then concentrate on the \textit{sampling version}, when the true statistical independence information is unknown and needs to be inferred from the observational data.

\subsubsection{Oracle version}
The oracle version of the PC algorithm assumes that the exact conditional probability distribution is known. It has been thoroughly studied in \cite{spirtes2000causation} and also briefly reviewed by \cite{colombo2014order}. 

The PC algorithm and PC-stable algorithm are illustrated together in Algorithm 1. The latter is built upon the former with only slight modification, thus we show them all together rather than separately. The indicator variable \textit{stable} decides whether PC or PC-stable will be used. 

The PC algorithm starts with a complete (fully-connected) undirected graph and includes two phases. In the first phase, the algorithm queries conditional independencies from the oracle and deletes the edge between two nodes in the DAG if they are independent conditioned on the third (possibly empty) set of variables -- sometimes referred to as the separation set. The algorithm also keeps a record of all separation sets as it progresses. At the end of the first phase, the algorithm outputs a skeleton graph. The second phase utilizes the separation set information to identify $v$-structures on the skeleton, then uses further four rules (R1 to R4) to logically infer the edge directions.


We briefly explain the steps of Algorithm 1. The PC algorithm is used when \textit{stable} is \texttt{False}. Line 1 defines the \textit{stage} $l$, which increases from $0$ to $n-2$, denotes the cardinality of the separation set $\mathcal{K}$. In Lines 3-4, for each two adjacent nodes connecting two nodes, we find a set of nodes $adj(G,i,j) \cap path(G, i, j)$ which is an intersection of their neighbours and the nodes on all paths connecting them. This is the minimal set that contains a possible separation set of nodes to make the two adjacent nodes  independent or conditionally independent. At stage $l$, if a pair of nodes $i, j$ has such set $adj(G,i,j) \cap path(G, i, j)$ having a cardinality of less than $l$, then this pair can be skipped. Line 6 iterate all size-$l$ subsets $X_{\mathcal{K}}$ of $adj(G,i,j) \cap path(G, i, j)$. At line 7, if $X_i \CI X_j \lvert X_{\mathcal{K}}$, then this edge is deleted from the graph $G$ (line 9), and the separation set $\mathcal{K}$ is added into the entry $(i,j)$ and $(j,i)$ of matrix $S$. Line 21 denotes early termination criteria. Notice that in line 2, the algorithm proceeds to the next stage only when the flag variable \textit{continue} is \texttt{True}. This indicator variable is initially set as \texttt{False}. In line 3-4, if no pair of nodes $i,j$ in $G$ has $|adj(G,i,j) \cap path(G, i, j)|$ larger than $l$, line 5 will never be executed, and lines 21-23 would be executed and the algorithm terminates. The reason for early termination is that, if none of the nodes $i,j$ has $|adj(G,i,j) \cap path(G, i, j)| > l$, then none of them would have $|adj(G,i,j) \cap path(G, i, j)| > l+1$, and there is no need to proceed to the next stage.

If \textit{stable} is \texttt{True}, then PC-stable algorithm is used. In this case, line 9 is skipped, and instead, line 11 is executed, and also lines 18-20. The PC-stable algorithm does not delete an edge immediately, but keeps them into a set $D$ and performs deletion all at once in the end of each stage. As discussed in the next section, delaying the deletion of edges to the end of each stage produces order-independent results in the sample (non-oracle) version. 

In Phase II, lines 2-4 identify all unshielded triples $X_i - X_k - X_j$ in the skeleton. The logic is that if the unshielded triple is $X_i \rightarrow X_k \rightarrow X_j$, $X_i \leftarrow X_k \leftarrow X_j$ or $X_i \leftarrow X_k \rightarrow X_j$, the $X_k$ must appear in the entry $S(i,j)$ (or $S(j,i)$, notice that $S$ is a symmetric matrix). On the contrary, if $X_k$ does not exist in $S(i,j)$, the unshielded triple must form a $v$-structure, $X_i \rightarrow X_k \leftarrow X_j$. The four rules in lines 7-10 is to infer directions for as more edges as possible, referred to as \textit{orientation propagation} \cite{glymour2019review}. For example, the Rule 1 rules out the possibility of $X_i \leftarrow X_j$, because it will make $X_k$ a $v$-structure node which should already be found at line 3,  Rule 2 would create a cyclic graph if it's  $X_i \leftarrow X_j$; etc.

In the next section, we discuss the PC and PC-stable algorithms when the exact conditional probability distribution is unknown, which we refer to as the \textit{Sampling Version} of the algorithm.

\begin{algorithm}[thbp]
\caption{\textbf{1.} PC \& PC-stable Algorithm (Oracle Version)- Phase I}
\label{alg:PC-I}
\textbf{Input:} \\ 
\hspace*{\algorithmicindent} $G$: Fully-connected undirected graph with $n$ vertices \\
\hspace*{\algorithmicindent} $S$: $n \times n$ square matrix \\
\hspace*{\algorithmicindent} $stable$: \texttt{False}: PC Algorithm, \texttt{True}: PC-stable Algorithm \\
\hspace*{\algorithmicindent} $D$: Empty set \\

\textbf{Output:} \\
\hspace*{\algorithmicindent} $G_{skel}$: Estimated skeleton \\
\hspace*{\algorithmicindent} $S$: Separation set \\
\begin{algorithmic}[1]
\FOR{each $l$ from $0$ to $n-2$}
\STATE continue = \texttt{False}
\FOR{each pair of adjacent vertices $i,j$ in $G$}
    \IF{$adj(G,i,j) \cap path(G, i, j) \geq l$} 
        \STATE continue = \texttt{True}
        \FOR{$X_\mathcal{K} \subseteq adj(G,i,j) \cap path(G, i, j)$ with $|\mathcal{K}| = l$}
            \IF{$X_i \CI X_j \lvert X_{\mathcal{K}}$}
                \IF{\textit{stable} is \texttt{False}} 
                    \STATE Delete edge $X_i- X_j$ in $G$
                \ELSE
                    \STATE Add edge $X_i-X_j$ to $D$
                \ENDIF
            \STATE Add $\mathcal{K}$ into $S(i, j)$ and $S(j, i)$
            \ENDIF
        \ENDFOR
    \ENDIF
\ENDFOR
\IF{stable is \texttt{True}}
    \STATE Delete all edges in $D$ from the graph $G$ and reset $D$ to empty
\ENDIF
\IF {continue is \texttt{False}}
\STATE Break the for-loop and goto line 21
\ENDIF
\ENDFOR

\STATE $G_{skel} = G$
\end{algorithmic}
\end{algorithm}

\begin{algorithm}[thbp]
\caption{\textbf{1.} PC \& PC-stable Algorithm - Phase II}
\label{alg:PC-II}
\textbf{Input:} \\ 
\hspace*{\algorithmicindent} Estimated skeleton $G_{skel}$, separation set $S$. \\
\textbf{Output:} \\ 
\hspace*{\algorithmicindent} CPDAG $G$. \\

\begin{algorithmic}[1]

\FOR{All pairs of non-adjacent vertices $X_i, X_j$ with common neighbor $X_k$}
    \IF{$k \notin S(i,j)$}
    \STATE Replace $X_i-X_k-X_j$ in $G$ with $X_i \rightarrow X_k \leftarrow X_j$
    \ENDIF
\ENDFOR
\STATE{Repeatedly apply the following rules until no more edges can be oriented}
\STATE \textbf{R1:} Change $X_j-X_k$ to $X_j \rightarrow X_k$ if $X_i \rightarrow X_j$, $X_i$ and $X_k$ are non-adjacent, 
\STATE \textbf{R2:} Change $X_i-X_j$ to $X_i \rightarrow X_j$ if $X_i \rightarrow X_k \rightarrow X_j$
\STATE \textbf{R3:} Change $X_i-X_j$ to $X_i \rightarrow X_j$ if there are two chains $X_i -X_k \rightarrow X_j$ and $X_i - X_l \rightarrow X_j$ such and $X_k$ and $X_l$ are non-adjacent,
\STATE \textbf{R4:} Change $X_i-X_j$ to $X_i \rightarrow X_j$ if there are two chains $X_i -X_k \rightarrow X_l$ and $X_k \rightarrow X_l \rightarrow X_j$ such that $X_k$ and $X_l$ are non-adjacent,
\end{algorithmic}
\end{algorithm}

\subsubsection{Sampling Version}
For random variables following multi-variant Gaussian distribution $X_{\mathcal{V}} \sim  \mathcal{N}(\boldsymbol{\mu}, \Sigma)$, the conditional independence is equivalent to the corresponding partial correlation coefficient being equal to zero \cite{lauritzen1996graphical}. However, in the sampling version of the algorithm, partial correlation coefficients need to be estimated from the data. Due to the existence of estimation error, it may not be exactly zero even when two variables are conditionally independent. Therefore, line 7 of Algorithm 1 is replaced by a statistical conditional independence test

\begin{center}
    line 7: \textbf{if } $|F_z(r_{ij\cdot \mathcal{K}})| \leq \frac{\Phi^{-1}(1-\alpha/2)}{\sqrt{N - |\mathcal{K}| - 3} }$ \textbf{ then}
\end{center}
where we denote $\rho_{ij \cdot \mathcal{K}}$ as the true partial correlation coefficient according to the exact probability distribution, and $r_{ij \cdot \mathcal{K}}$ be an estimation from the sample data. The Fisher's $z$-transformation is applied on $r_{ij \cdot \mathcal{K}}$,
\begin{equation}
    F_z(r_{ij \cdot \mathcal{K}}) = \arctanh(r_{ij \cdot \mathcal{K}}) = \frac{1}{2}\ln{\frac{1+r_{ij \cdot \mathcal{K}}}{1-r_{ij \cdot \mathcal{K}}}}.
\end{equation}

It has been proved that $F_z(r_{ij \cdot \mathcal{K}})$ quickly approaches a Gaussian distribution $\mathcal{N}$ $(\frac{1}{2}\ln{\frac{1+\rho_{ij \cdot \mathcal{K}}}{1-\rho_{ij \cdot \mathcal{K}}}}, \frac{1}{N-|\mathcal{K}|-3})$ for any $\rho_{ij \cdot \mathcal{K}}$ as the sample size $N$ approaches infinity \cite{fisher1992statistical}. Therefore, we set the null hypothesis as $H_0(i, j \lvert \mathcal{K}): \rho_{ij \cdot \mathcal{K}} = \rho_0 = 0$ and the alternative hypothesis as $H_1(i, j \lvert \mathcal{K}): \rho_{ij \cdot \mathcal{K}} = \rho_1 \neq 0$, and $H_0$ is rejected with significant level $\alpha$ if $|F_z(r_{ij \cdot \mathcal{K}})| > \frac{\Phi^{-1}(1-\alpha/2)}{\sqrt{N-|\mathcal{K}| - 3}}$, where $\Phi$ is the cumulative distribution function of a standard normal distribution  \cite{fisher1924distribution}.

The order-dependence feature of the PC algorithm is mainly due to the possible erroneous existed in the statistical (conditional) independence tests when applied on real datasets. The PC and and PC-stable algorithm yield the same result if the true conditional independence information is given. However, on real-world observational data the exact conditional probability distribution is unknown, and estimate partial correlation coefficient from finite data samples may introduce bias. The error made at early stages in the PC algorithm - deleting an edge that should exist or keeping an edge that does not exist in the true causal graph - will be propagated to subsequent stages and incur more errors. This cascading effect in the PC algorithm will result in a varying result skeleton. 

The PC-stable algorithm has been proved to be an improvement for this order-dependent issue, by the following theorem:
 
\begin{theorem}[Colombo and Maathius, Theorem 3]
The skeleton of the PC-stable algorithm is order-independent when the conditional independence needs to be inferred from data. \cite{colombo2012modification}
\end{theorem}

\subsection{Performance Bottleneck}
The PC and PC-stable algorithm are computationally intensive, in that the two algorithms need to perform an exhaustive search for the conditioned nodes which is of super-exponential complexity. The number of conditional independence tests is $\binom{n}{2} \cdot$ $\sum_{i=0}^k\binom{n-2}{i}$ in the worst case, with an upper bound of 
\begin{equation}\label{eq:pc_upper_bound}
    \frac{n(n-1)^{k+1}}{2}
\end{equation}
using the binomial theorem. The integer $n$ is the dimension (number of nodes) of the causal graph, and integer $k \in [0, n)$ is the maximum degree among all nodes. This time complexity makes the algorithm infeasible on large graphs, which promotes us to seek a speed-up solution.

Note that we did not include the time complexity of calculating $Path(G, i, j)$ in formula \ref{eq:pc_upper_bound}, the calculation of which needs the Depth-first search (DFS) or Breath-first search (BFS) algorithm for each pair of nodes $u, v$ where $v \in neighbour(i)$ and $u \in neighbour(j)$. The DFS or BFS requires $O(|\mathcal{V}| +|\mathcal{E}|)$ computational time. 

\section{Reverse Order Pruning PC Algorithm}
We propose a reverse order pruning PC algorithm proceeding in the order of decreasing conditional set size $l$. First, we define the \textit{$v$-structure nodes} of two nodes $X_i$ and $X_j$ to be $X_{\mathcal{W}_{ij}}$, where 
\begin{equation}
    \mathcal{W}_{ij} = \{k \in \mathcal{V} \text{ s.t. } X_i \rightarrow X_k \leftarrow X_j\}, \text{ for } i, j \in \mathcal{V},
\end{equation}
in other words, $X_{\mathcal{W}_{ij}}$ includes all nodes $X_{k}$ that has the exact structure $X_i \rightarrow X_k \leftarrow X_j$. Moreover, we introduce the following notations.

\begin{definition}
We use $de(X_i)$ to denote the descendants of a node $X_i, i \in \mathcal{V}$ in a DAG $G(\mathcal{V}, \mathcal{E})$. 
Furthermore, we define set $de(X_{\mathcal{K}})$ to be the union of subsequent descendants of all nodes in $X_{\mathcal{K}}$, 
\begin{equation}
   de(X_{\mathcal{K}}) \triangleq \cup_{i \in \mathcal{K}}de(X_i).
\end{equation}
\end{definition}

The following proposition serves as theoretical support for the speed-up of our new algorithm.

\begin{proposition}\label{prop:prop1}
Let $P(X_{\mathcal{V}})$ be the probability distribution that satisfies the causal Markov and faithfulness conditions with respect to the DAG $G(\mathcal{V}, \mathcal{E})$, and assume that all exact conditional independence information is known. If $X_i \CI X_j \vert X_{\mathcal{K}_{ij}^{\min}}$ where $X_{\mathcal{K}_{ij}^{\min}}$ is one of the minimal sets of variables rendering $X_i$ and $X_j$ independent (that is, a separation set of the smallest cardinality among all subsets of $X_{\mathcal{V} \backslash \{i, j\}}$), then $X_i \CI X_j \lvert X_{\mathcal{K}_{ij}^{\max}}$ where set $\mathcal{K}_{ij}^{\max} \subseteq \mathcal{V} \backslash \{i, j\}$ is one of the maximal sets rendering $X_i$ and $X_j$ independent (that is, a separating set of the largest cardinality among all subsets of $X_{\mathcal{V} \backslash \{i, j\}}$).

One way to construct $X_{\mathcal{K}_{ij}^{\max}}$ is to follow the below steps

\hfill

(Step 1) Let $\mathcal{K}_{ij}^{\max}$ include all variables in $\mathcal{K}_{ij}^{\min}$, that is,  $\mathcal{K}_{ij}^{\min} \subseteq \mathcal{K}_{ij}^{\max}$,

\hfill

(Step 2) Let $\mathcal{K}_{ij}^{\max}$ include as more as the other variables (excluding $i$ and $j$), satisfying the following two rules:

\hfill

(Rule 2-1) $X_{\mathcal{K}_{ij}^{\max}}$ does not include any variable of $X_{\mathcal{W}_{ij}}$ and $de(X_{\mathcal{W}_{ij}})$,

\hfill

(Rule 2-2) If $X_{\mathcal{K}_{ij}^{\max}}$ includes a collider $X_k$ on the path between $X_i$ and $X_j$, but not a $v$-structure node of $X_i$ and $X_j$ (a node in $X_{\mathcal{W}_{ij}}$), then for any directed path between $X_i, X_j$ that pass through $X_k$, $X_{\mathcal{K}_{ij}^{\max}}$ should include all the other non-collider variables on that path, exclude those in $X_{\mathcal{W}_{ij}}$ and $de(X_{\mathcal{W}_{ij}})$. For all such paths, if for at least one path there is no such non-collider variables, then $X_k$ cannot be included in $X_{\mathcal{K}_{ij}^{\max}}$.

\hfill

Both the set $X_{\mathcal{K}_{ij}^{\min}}$ and $X_{\mathcal{K}_{ij}^{\max}}$ may not be unique.
\end{proposition}

The proof of the above statement is in Appendix \ref{appendix:prop:prop1}.

In the next section, we demonstrate the benefits of the reverse order pruning using a motivational example, then follows by our proposed algorithm and detailed analysis.

\subsection{A Motivational Example}
Consider an 4-variable DAG with a multi-variant Gaussian probability distribution, and assume all exact conditional probability information is known. The true underlying DAG is shown in Figure \ref{fig:graph}.

\begin{figure}[thbp]
\centering
\includegraphics[width=3cm]{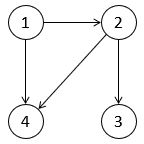}
\caption{The true DAG with 4 nodes.}
\label{fig:graph}
\end{figure}

The PC algorithm proceeds by increasing $l$ such that when it reaches $l=4$, and our algorithm proceeds in an reverse order. The Fig. \ref{fig:motive_example_PC} and Fig. \ref{fig:motive_example_PC_reverse} gives a full picture of the CI tests performed in the Phase I of the conventional PC algorithm and the proposed algorithm to learn the graph skeleton. The proposed algorithm performs 14 CI tests compared to the 17 in the conventional PC algorithm, which is 3 less.

\begin{figure}[thbp]
\centering
\includegraphics[width=0.5\textwidth]{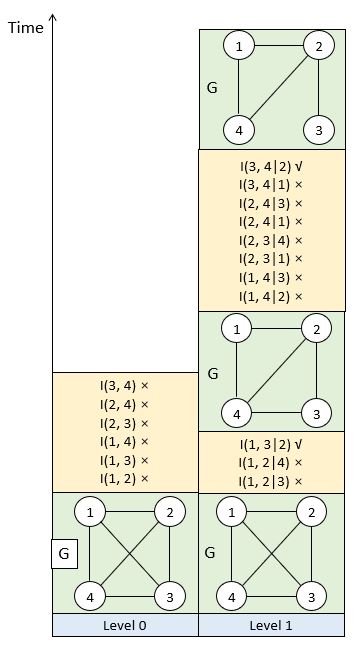}
\caption{An example of execution of PC algorithm. For better
readability, we use the term $i$ instead of $X_i$, and $I(i, j | k)$ to denote the CI test of $X_i, X_j$ conditioned on $X_k$.}
\label{fig:motive_example_PC}
\end{figure}

\begin{figure}[thbp]
\centering
\includegraphics[width=0.7\textwidth]{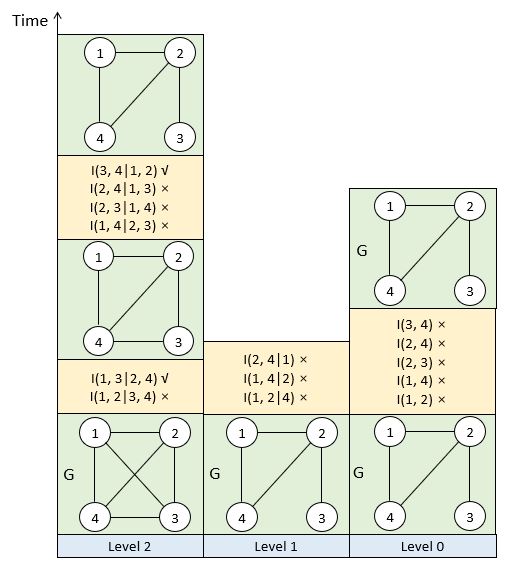}
\caption{An example of execution of PC-reverse algorithm. For better
readability, we use the term $i$ instead of $X_i$, and $I(i, j | k)$ to denote the CI test of $X_i, X_j$ conditioned on $X_k$.}
\label{fig:motive_example_PC_reverse}
\end{figure}

The Fig. \ref{fig:graph} embeds two conditional independence relations: $X_1 \CI X_3 \lvert X_2$, $X_3 \CI X_4 \lvert X_2$. A critical observation is that $X_1 \CI X_3 \lvert \{X_2, X_4\}$, $X_3 \CI X_4 \lvert \{X_1, X_2\}$, that is, if we incorporate some other variables in the conditional set, as long as this redundant variable does not open another path of information flow (collider or decsendents of collider), then the two originally independent variables still remains independent. 
The advantage of reverse order (stage) edge pruning is more significant when the graph is large. The intrinsic reason is, for finding the true conditional independence relation $X_i \CI X_j \lvert \{X_1, \ldots, X_k\}$, the conventional PC algorithm needs to traverse all the subsets (which is a power set) of $\{X_1, \ldots, X_k\}$ until reaches $X_{\mathcal{K}}=\{X_1, \ldots, X_k\}$ at the stage-$k$. None of the subset in the power set will not make $X_i \CI X_j$, and edge $X_i -X_j$ will not be deleted. On the contrary, our proposed algorithm can potentially finds all independence relations in only one stage running of the algorithm (see this example), by including all the other $n-2$ variables into the conditional set. The condition is that the 'redundant' variables does not include open another information flow path between $X_i$ and $X_j$. 

Another advantage early deletion is that, early deletion of edges make the graph sparser and reduces the number of CI tests in the future stages. In Fig. \ref{fig:motive_example_PC_reverse}, we see that even though the all independence relations has been found at stage-2, the CI tests in stage-1 and stage-0 are fewer than in Fig. \ref{fig:motive_example_PC}, because each node has less number of neighbors to test as the graph quickly becomes sparser in the beginning.

\subsection{Reverse Order Pruning PC Algorithm}
In this section, we formally present the proposed reverse order pruning PC algorithm, and the precise conditions under which the proposed algorithm is faster than the conventional PC.

\subsubsection{Find the Skeleton}
Algorithm 2 is the proposed algorithm (Phase I). The difference from the conventional PC algorithm (Algorithm 1) is the order of the stage $l$, and also the lines 21-23 of Algorithm 1. Conventional PC algorithm applies an additional early termination policy, the algorithm stops at stage $l$ when the DAG reaches a state where no pair of node has larger than $l$ neighbors. However, this will never happen in our proposed algorithm since $l$ is decreasing. Therefore, our proposed algorithm will always proceed to the end (finishes stage 0).

The following proposition states at which stage $l$ an edge would be deleted, which will later be useful for proving the main theorem.

\begin{proposition}\label{prop:whichstage}
Let the distribution of $X_{\mathcal{V}}$ be faithful to the DAG $G(\mathcal{V}, \mathcal{E})$, and assume the exact conditional probability distribution is given for all $X_i, X_j$ given $X_{\mathcal{K}} \subseteq X_{\mathcal{V}} \backslash \{X_i,X_j\}$. 
\begin{itemize}
    \item If $X_i$ and $X_j$ are not adjacent in the true DAG, the edge $X_i - X_j$ will be deleted when Algorithm 2 reaches stage $l = |\mathcal{K}_{ij}^{\max}|$.
    \item If $X_i$ and $X_j$ are adjacent in the true DAG, Algorithm 2 will keep testing the edge $X_i - X_j$ until finishes the final stage $l = 0$.
\end{itemize}
\end{proposition}

The proof is given in appendix \ref{appendix:prop:whichstage}. 

The speed escalation of the proposed algorithm compared to the conventional PC is as follows.

\begin{corollary}\label{coro:timecomplexity}
For any two nodes $X_i$ and $X_j$ that are not adjacent in the true DAG, suppose $\mathcal{K}_{ij}^{\min}$ is one of the minimal subset of $\mathcal{V}$ rendering $X_i,X_j$ independent, and $\mathcal{K}_{ij}^{\max}$ is one of the maximum set in Proposition \ref{prop:prop1}, then
\begin{equation}
    \frac{Q_{ij}(n,G)}{Q_{ij}^r(n,G)} = \frac{(n-1)^{|\mathcal{K}_{ij}^{\min}|}}{(n-1)^{n-2-|\mathcal{K}_{ij}^{\max}|}} = (n-1)^{|\mathcal{K}_{ij}^{\min}| + |\mathcal{K}_{ij}^{\max}|-n+2},
\end{equation}
where $Q_{ij}(n,G)$ is the number of CI queries for edge $X_i - X_j$ in PC algorithm, $Q_{ij}^r(n,G)$  is the number of CI queries for edge $X_i - X_j$ in reverse order pruning PC algorithm, and both are functions of the number of variables $n$ and the graph structure $G$.
\end{corollary}

The proof is in Appendix \ref{appendix:coro:timecomplexity}. Corollary \ref{coro:timecomplexity} indicates our algorithm has speed gain on testing an edge $X_i - X_j$ if $n-2-|\mathcal{K}_{ij}^{\max}| < |\mathcal{K}_{ij}^{\min}|$, that is, the number of nodes that when conditioned on will open a path between $X_i,X_j$ is fewer than the minimal set $\mathcal{K}_{ij}^{\min}$.

\begin{algorithm}[thbp]
\caption{\textbf{2.} Reverse order pruning PC (phase II omitted: same with Algorithm 1}
\label{alg:Fast PC-I}
\textbf{Input:} \\ 
\hspace*{\algorithmicindent} $G$: Fully-connected undirected graph with $n$ vertices \\
\hspace*{\algorithmicindent} $S$: $n \times n$ square matrix \\
\hspace*{\algorithmicindent} $stable$: \texttt{False}: Reverse order pruning PC algorithm, \texttt{True}: Reverse order pruning PC-stable algorithm \\
\hspace*{\algorithmicindent} $D$: Empty set \\

\textbf{Output:} \\
\hspace*{\algorithmicindent} $G_{skel}$: Estimated skeleton \\
\hspace*{\algorithmicindent} $S$: Separation set \\
\begin{algorithmic}[1]
\FOR{each $l$ from $n-2$ to $0$}
\FOR{each pair of adjacent vertices $i,j$ in $G$}
    \IF{$adj(G,i,j) \cap path(G, i, j) \geq l$} 
    \FOR{$X_\mathcal{K} \subseteq adj(G,i,j) \cap path(G, i, j)$ with $|\mathcal{K}| = l$}
        \IF{$X_i \CI X_j \lvert X_{\mathcal{K}}$}
            \IF{\textit{stable} is \texttt{False}} 
                \STATE Delete edge $X_i- X_j$ in $G$
            \ELSE
                \STATE Add edge $X_i-X_j$ to $D$
            \ENDIF
        \STATE Add $\mathcal{K}$ into $S(i, j)$ and $S(j, i)$
        \ENDIF
    \ENDFOR
    \ENDIF
\ENDFOR
\IF{stable}
    \STATE Delete all edges in $D$ from the graph $G$ and reset $D$ to empty
\ENDIF
\ENDFOR

\STATE $G_{skel} = G$
\end{algorithmic}
\end{algorithm}

Similar as the conventional PC algorithm, the sampling version of the proposed algorithm is to replace line 6 of Algorithm 2 with 
\begin{center}
    \textbf{if } $F_z(r_{ij\cdot \mathcal{K}}) \leq \frac{\Phi^{-1}(1-\alpha/2)}{\sqrt{N - |\mathcal{K}| - 3} }$ \textbf{ then}
\end{center}
when the exact conditional independence information is unknown.

\subsubsection{Infer the CPDAG}
This is the Phase II of the proposed algorithm and is the same with conventional PC algorithm, therefore, we omit it in Algorithm 2. 

The proposed algorithm differs from the conventional PC algorithm in the matrix $S$, since the entry of $i$-th row and $j$-th column of matrix $S$ now contains the separation set $\mathcal{K}_{ij}^{\max}$ rather than $\mathcal{K}_{ij}^{\min}$. We will prove the proposed algorithm and PC algorithm will produce the same CPDAG in Phase II even though they produce different matrices $S$ in phase I.

\begin{proposition}\label{prop:prop2}
For the oracle version, if the conventional PC algorithm and reverse order pruning PC algorithm generate the same skeleton after the first phase, and every entry in the matrix $S$ of PC algorithm and the corresponding entry in the matrix $S'$ of reverse order pruning PC algorithm satisfy the condition in Proposition \ref{prop:prop1}, then they will produce the same CPDAG in the second phase.
\end{proposition}

The proof is rather simple once we understand how the information in set $S$ is used in the Phase II. The set $S$ is only useful for changing structure $X_i - X_k - X_j$ to $X_i \rightarrow X_k \leftarrow X_j$ in the skeleton -- that is, if $X_k$ is not in $S(i,j)$ then $X_k$ must be a collider in such a structure. Therefore, as long as the $(i, j)$-entry of matrix $S$ does not include any collider, Phase II of the algorithm would produce the exact same result. Indeed, this requirement is guaranteed by Proposition \ref{prop:prop1}.

\section{Algorithm Consistency on High Dimensional Dataset}
In this section, We will prove the resulting skeleton of the proposed algorithm is consistent on high dimensional data. 

\subsection{Find the Skeleton}
We use $n=n_N$ to represent the number of random variables as a function of the sample size $N$, and the $n_N$ random variables are indexed by $\mathcal{V} = \{1, \ldots, n_N\}$, thus $X_{\mathcal{V}}$. Assume $X_{\mathcal{V}}$ follows a multi-variant Gaussian distribution $P(X_{\mathcal{V}})$. Define $v_N$ as the smallest $\mathcal{K}_{ij}^{\max}$ among all the pairs of nodes $X_i,X_j$ that are not adjacent in the true DAG, 
\begin{equation*}
     v_N \triangleq \displaystyle{\min_{i,j} |\mathcal{K}_{ij}^{\max}|}, \; \forall i \neq j \in \text{ s.t. $X_i, X_j$ are non-adjacent in the true DAG},
\end{equation*}
where $\mathcal{K}_{ij}^{\max}$ is defined in Proposition \ref{prop:prop1}. The $v_N$ is determined solely by the DAG structure, and we allow it to be a function of the sample size $N$. Obviously, $v_N$ can not exceed $n-2$ in any DAG.

If the following four conditions are satisfied, then we have Theorem \ref{them:skeletonconsistency}.

\hfill

(C1) The underlying distribution $P(X_{\mathcal{V}})$ of the DAG $G(\mathcal{V}, \mathcal{E})$ is faithful to the graph,

\hfill

(C2) The number of random variables (nodes in the DAG) grows no faster than a sublinear function of $N$, $n_N = O(N^d)$, $0 < d < 1$,

\hfill

(C3) The absolute value of all partial correlations coefficients $\rho_{ij \cdot \mathcal{K}}$ are bounded by 
\begin{align*}
    & \inf \{|\rho_{ij \cdot \mathcal{K}}| \text{ for all } \rho_{ij \cdot \mathcal{K}} \neq 0\} \geq c_N = O(N^{-b}), 0 < b < (1-d)/2, \\
    & \sup \{|\rho_{ij \cdot \mathcal{K}}|\} \leq M < 1.
\end{align*}
where $d$ is defined in (C2),

\hfill

(C4) $v_N = O(N^e)$, $0<e<d$ where $d$ is defined in (C2). 

\begin{theorem}\label{them:skeletonconsistency}
Assume the conditions (C1)-(C4) are satisfied, and denote $\hat{G}_{skeleton}$ as the skeleton of the proposed algorithm's Phase I result, and $G_{skeleton}$ be the true skeleton of the graph, then there exists a significance level $\alpha \rightarrow 0$ such that
\begin{equation}
    P(\hat{G}_{skeleton} = G_{skeleton}) = 1 - O(e^{-CN^{1-2b}}) \rightarrow 1 \text{ as } N \rightarrow \infty,
\end{equation}
where $0 < C < \infty$ is some constant and $b$ is defined in (C3).
\end{theorem}

The proof is given in appendix \ref{appendix:them:skeletonconsistency} which takes the same spirit as Theorem 1 in \cite{kalisch2007estimating} with some differences, due to the assumption difference. 

The condition (C1) is a basic faithfulness condition that is assumed by most structure-learning algorithms such as PC algorithm. The condition (C2) is to ensure the sample size is significantly larger than the number of random variables, $N \gg n$. In condition (C3), the restriction on $c_N$ is quite slight since as $N \rightarrow 0$, $c_N$ approaches zero. The upper bound $M$ is to ensure the partial correlation coefficient is not too large -- that is, 1.0. Notice $M$ can be any constant between $c_N$ and 1, and we do not require it to be a function of $N$.  The $v_N$ in condition (C4), based on its definition, is also the last stage that our proposed algorithm would delete an edge (in the oracle version), since all non-existing edges in the true DAG would be successfully deleted during stages $l = n-2$ to $v_N$. The condition (C4) assumes $v_N$ also grows slower than a linear function of $N$, thus a sublinear function. Notice that the maximum of $v_N$ should not exceed the total number of variables $n$, which is guaranteed by $e < d$. 

\textbf{Comparison with Theorem 1 of \cite{kalisch2007estimating}} The condition (C1) and (C3) are similar to the conditions (A1) and (A4), respectively, in \cite{kalisch2007estimating}. The conditions (A2) and (A3) in \cite{kalisch2007estimating} are not required, they are

\hfill

(A2) The dimension $n_N = O(N^a)$ for some $0 \leq a < \infty$,

\hfill

(A3) The maximal number of neighbors in the DAG $G$ is denoted by $q_N = \max_{1 \leq j \leq n_N}|adj(G, j)|$, with $q_N = O(N^{1-b})$ for some $0 < b \leq 1$.

\hfill

Essentially, (A2) is replaced by (C2) in our algorithm, which assumes the dimension of the graph should grow no faster than a linear function of sample size $N$. (A3) is the sparsity requirement which restricts the maximal number of neighbors a node can have. It is not required for the algorithm. Instead, we assume (C4), which states that the minimum $\mathcal{K}_{ij}^{\max}$ should grow slower than the dimension $n_N$.

\subsection{Infer CPDAG}
As it has been proved by Theorem \ref{them:skeletonconsistency}, the estimated skeleton is the same with the true skeleton under conditions (C1)-(C4). Together with Proposition \ref{prop:prop2}, we can prove the Phase II will produce the true CPDAG. This is because the second phase does not include any statistical tests but only uses the separation matrix $S$ to infer the edge direction on the skeleton, which is deterministic. Therefore, we can prove the following theorem.

\begin{theorem}\label{them:cpdagconsistency}
Assume the conditions (C1)-(C4) are satisfied, and denote $\widehat{G}_{CPDAG}$ is the CPDAG of the proposed algorithm Phase II's result, and $G_{CPDAG}$ be the true CPDAG of the graph produced by the oracle version of the proposed algorithm, then there exists a significance level $\alpha \rightarrow 0$ such that
\begin{equation}
    P(\widehat{G}_{CPDAG} = G_{CPDAG}) = 1 - O(e^{-CN^{1-2b}}) \rightarrow 1 \text{ as } N \rightarrow \infty,
\end{equation}
where $0 < C < \infty$ is some constant and $b$ is defined in (C3).
\end{theorem}

The proof is given in appendix \ref{appendix:them:cpdagconsistency}. 

\subsection{Statistical Power Analysis}
In this section, we analyze the statistical power of the partial conditional independence test. Before we dive into the details, we think it is necessary to re-illustrate that the ``adding more variables into $\mathcal{K}$ is not blindly performed. Specifically,
\begin{itemize}
\item  If $\rho_{ij \cdot \mathcal{K}}=0$, then $\rho_{ij \cdot \mathcal{K}'}=0, \mathcal{K}' \subset \mathcal{K}$ as long as $\mathcal{K}' \backslash \mathcal{K}$ does not include any collider or its descendants. To put it formally,  $\mathcal{K}'$ can be any subset of $\mathcal{K}_{ij}^{max}$ in Proposition 5. 
\item If $\rho_{ij \cdot \mathcal{K}} \neq 0$, then $\rho_{ij \cdot \mathcal{K}'} \neq 0, \mathcal{K}' \subset \mathcal{K}$ as long as $\mathcal{K}'$ does not block all paths of information flow between $X_i$ and $X_j$.
\end{itemize}
The keynote is we cannot add variables to $K$ such any of them changes $\rho_{ij \cdot \mathcal{K}}$ from 0 to nonzero or vice versa. 

Since the true correlation coefficient $\rho_{ij \cdot \mathcal{K}} \neq 0$ is never observed, and we use $r_{ij\cdot \mathcal{K}}$ to estimate it, and CI test is to compare $F_z(r_{ij\cdot \mathcal{K}}) \leq \frac{\Phi^{-1}(1-\alpha/2)}{\sqrt{N - |\mathcal{K}| - 3} }$. Essentially, our analysis reduces to analyze if $\mathcal{K}$ becomes $\mathcal{K}' (\mathcal{K} \subset  \mathcal{K}')$, whether the above CI test (inequality relationship) will be affected or not, which directly related to the statistical power of the test (see the top figure in Fig. \ref{fig:powerAnalysis}). We discuss the following two cases separately,

\begin{equation}
F_z(r_{ij \cdot \mathcal{K}}) =
\begin{cases}
  \mathcal{N}(0, \frac{1}{N-|\mathcal{K}|-3}), & \text{if}\ \rho_{ij \cdot \mathcal{K}}=0 \\
  \mathcal{N}(\frac{1}{2}\ln{\frac{1+\rho_{ij \cdot \mathcal{K}}}{1-\rho_{ij \cdot \mathcal{K}}}}, \frac{1}{N-|\mathcal{K}|-3}), & \text{if}\ \rho_{ij \cdot \mathcal{K}} \neq 0 
\end{cases}
\end{equation}

\subsubsection{Case \texorpdfstring{$\rho_{ij \cdot \mathcal{K}}=0$}{lg}} 
In this case,  $X_i \CI X_j \lvert X_{\mathcal{K}}$, we will check when performing statistical independence tests whether $X_i$ and $X_j$ are still independent given $X_{\mathcal{K}'}$, and whether it will make any differences on the statistical power.

Since the mean of the distribution is the same for two cases, the two distributions $F_z(r_{ij \cdot \mathcal{K}})$ and $F_z(r_{ij \cdot \mathcal{K}'})$ are shown as the blue solid and blue dashed line in Fig. \ref{fig:powerAnalysis} bottom left.  Under the assumption that the sample size $N$ is of a higher order of the number of variables $n$, we have
\begin{equation}
    |\mathcal{K}| \leq n_N  = O(N^d), 0 < d < 1.
\end{equation}
The dominant term in the denominator is $N$. Therefore, changing $\mathcal{K}$ has asymptotically no influence on the distribution variance and the threshold. This shows in the figure as two very close distributions (blue solid vs. blue dashed). We conclude if $X_i$ and $X_j$ are independent given  $X_{\mathcal{K}}$, adding more variables into $X_{\mathcal{K}}$ does not affect the statistical power.

\begin{figure}[tb]
\centering
\includegraphics[width=12cm]{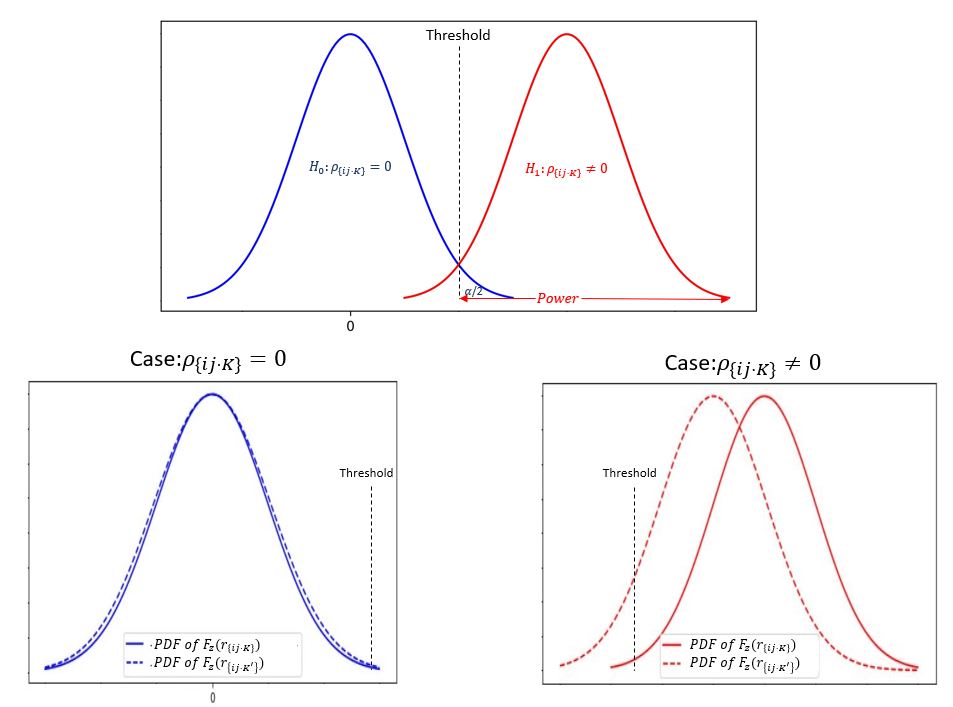}
\caption{The illustration of statistical power analysis when performing statistical tests.}
\label{fig:powerAnalysis}
\end{figure}

\subsubsection{Case \texorpdfstring{$\rho_{ij \cdot \mathcal{K}} \neq 0$}{lg}} 
In this case,  $X_i \nCI X_j \lvert X_{\mathcal{K}}$, we will check if $X_i \CI X_j$ are still independent given $X_{\mathcal{K}'}$, and analyze the statistical power difference of the two CI tests. 

For the same reason, the distribution variance and threshold are asymptotically not influenced. However, increasing set $\mathcal{K}$ has a noticeable influence on the mean of the distribution, which will cause the distribution to be shifted (to left), Fig. \ref{fig:powerAnalysis} bottom right. In Section 6, we will perform simulation experiment to analyze the loss of statistical power in this case.

\section{Parallelization and GPU Acceleration}
We propose a parallellization strategy for the proposed reverse order pruning PC algorithm. The Pearson correlation coefficient between real random variables can be computed from the correlation matrix, and parallellization can be realized by leveraging tensor operations. 

For $n$ random variables, we denote the $n$-by-$n$ correlation matrix as $corr = \{\rho_{ij}\}$, where the entry at the $i$-th row and $j$-th column denotes the correlation coefficient of variable $X_i$ and $X_j$. We denote the $b$ partial correlations to be computed parallelly as coefficient of $X_{i_p}, X_{j_p}$ given $X_{\mathcal{K}_p}$ where $\mathcal{K}_p = \{k_{p1}, \ldots, k_{pl}\}$, for $p = 1, \ldots, b$. 

The three dimensional tensor $\mathbf{H}^{b\times 2 \times 2}$ can be derived from the correlation matrix $corr$ and $B$, 
\begin{equation}
    \mathbf{H}^{b\times 2 \times 2} = \mathbf{H}_0^{b\times 2 \times 2} - \mathbf{H}_1^{b\times 2 \times l} (\mathbf{H}_2^{b\times l \times l})^{-1} (\mathbf{H}_1^{b\times 2\times l})^T,
\end{equation}
where
\begin{align}
    & \mathbf{H}_0^{b\times 2 \times 2}(p,:,:) = \begin{bmatrix}
    corr(i_p, i_p), & corr(i_p, j_p) \\
    corr(j_p, i_p),, & corr(j_p, j_p)
    \end{bmatrix},  p=1,\ldots, b,
    \\
    & \mathbf{H}_1^{b\times 2 \times l}(p,:,:) = \begin{bmatrix}
    corr(i_p, k_{p1}), & \ldots, &  corr(i_p, k_{pl}) \\
    corr(j_p, k_{p1}), & \ldots, &  corr(j_p, k_{pl})
    \end{bmatrix},  p=1,\ldots, b,
    \\
    & \mathbf{H}_2^{b\times l \times l}(p,:,:) =\begin{bmatrix}
    corr(k_{p1}, k_{p1}),  & \ldots, &  corr(k_{p1}, k_{pl}) \\
    \vdots & \ddots &  \vdots \\
    corr(k_{pl}, k_{p1}), & \ldots, &  corr(k_{pl}, k_{pl})
    \end{bmatrix},  p=1,\ldots, b.
\end{align}

We denote the matrix $\mathbf{H}$ as 
\begin{equation}
    \mathbf{H}^{b\times 2 \times 2}(p,:,:) = \begin{bmatrix}
    h_{p,11}, & h_{p,12} \\
    h_{p,21}, & h_{p,22}
    \end{bmatrix}, p=1,\ldots, b.
\end{equation}
The partial correlation coefficient can be calculated by scaling $\mathbf{H}$'s diagonal entries to 1 on each $p$-slice, and the off-diagonal element is equal to the partial correlation coefficient.  The $p$-th partial correlation coefficient can be read off from the $p$-th entry of vector $\boldsymbol{\rho}$,
\begin{equation}
    \boldsymbol{\rho}(p) = \frac{h_{p, 12}}{\sqrt{h_{p,11}h_{p,22}}}, p=1,\ldots, b.
\end{equation}
An illustration of the above operations is in Figure \ref{fig:tensor}. 

\begin{figure}[tb]
\centering
\includegraphics[width=10cm]{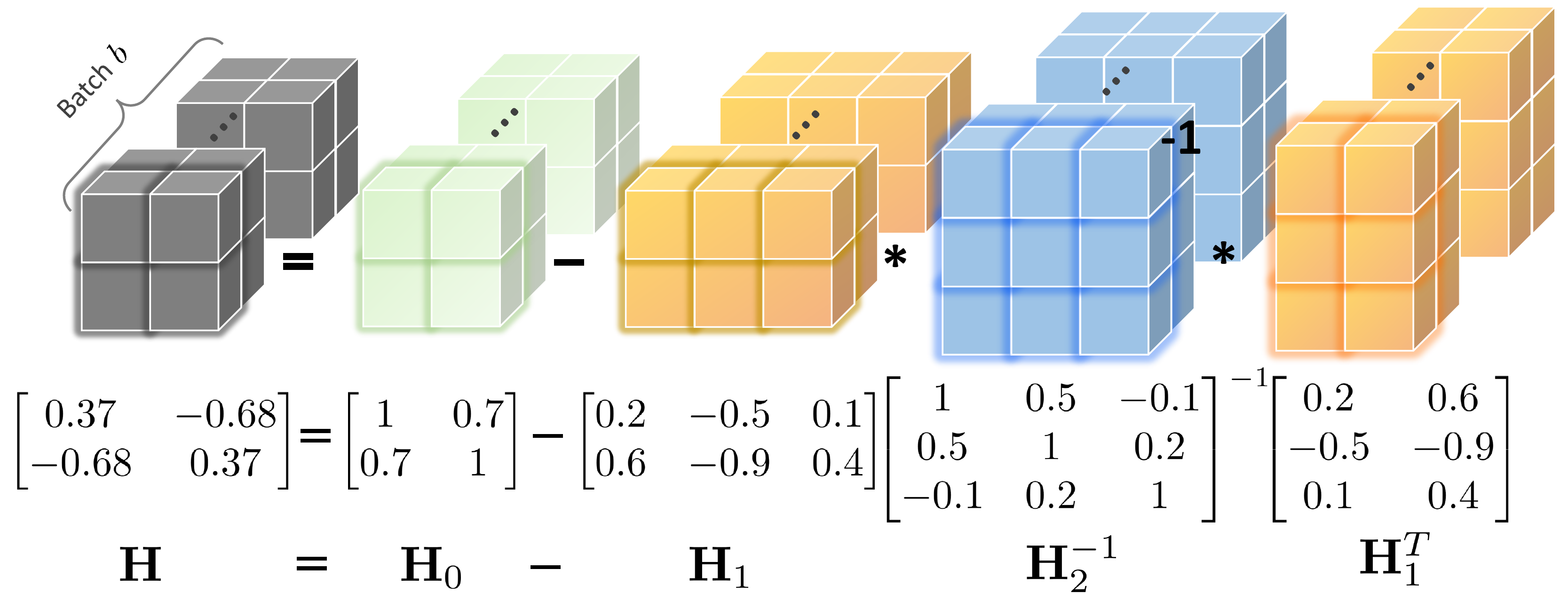}
\caption{An example of tensor computation for correlation coefficients.}
\label{fig:tensor}
\end{figure}


In implementation, we use the multidimensional indexing using multidimensional index arrays trick, and the $\mathbf{H}_0, \mathbf{H}_1, \mathbf{H}_2$ can each be constructed using one line instruction, instead of using the cumbersome for-loop construction. Specifically, we stack all the ($i$, $j$, $\mathcal{K}$) associated with the $b$ partial correlation coefficients into a matrix $B^{b \times (2+l)}$,
\begin{equation}
    B^{b\times (2+l)} = \begin{bmatrix}
    i_1, j_1, k_{11}, \ldots, k_{1l} \\
    \vdots \\
    i_p, j_p, k_{p1}, \ldots, k_{pl} \\
    \vdots \\
    i_b, j_b, k_{b1}, \ldots, k_{bl}
    \end{bmatrix}, p = 1, \ldots, b
\end{equation}
and use $B$ as an indexing matrix into the matrix $corr$ to construct $\mathbf{H}_0, \mathbf{H}_1, \mathbf{H}_2$.

The parallel reverse order pruning PC algorithm is in Algorithm 3. A critical difference with Algorithm 2 is the lines 9-12 and 17-19. In line 9, instead of computing partial correlation coefficients serially, we store the associated index into a matrix $B$ with pre-allocated size $b\times (2+l)$. When $B$ is fulfilled, the algorithm use $B$ as indices into $corr$ to form tensors and perform batch-wise CI tests. The information of which edge should be deleted is saved as well as the conditional set information. Then the graph $G$ is updated as well as the separation set $S$. Line 11 resets the batch $B$ and $m$ to prepare for the next iteration. At the last stage of the entire algorithm, the number of tests may not be enough to form a complete batch, therefore lines 9-11 are not executed, instead a semi-batch CI-tests is performed in lines 17-19. The above parallellization can be further accelerated by GPU.

\begin{algorithm}[thbp]
\caption{\textbf{Parallelization Version for Algorithm 2}}
\label{alg:A}
\textbf{Input:} \\ 
\hspace*{\algorithmicindent} $G$: Fully-connected undirected graph with $n$  vertices \\
\hspace*{\algorithmicindent} $S$: $n \times n$ square matrix \\

\textbf{Output:} \\
\hspace*{\algorithmicindent} $G_{skel}$: Estimated skeleton \\
\hspace*{\algorithmicindent} $S$: Separation set \\

\begin{algorithmic}[1]
\FOR{each $l$ from $n-2$ to $0$}
\STATE \textbf{Initialize}: Batch size $b$, and matrix $B$ with dimension $b \times (2 + l)$ \\
\STATE Set $m = 0$
    \FOR{each pair of adjacent vertices $i,j$ in $G$}
        \IF{$adj(G,i,j) \cap path(G, i, j) \geq l$} 
        \FOR{$X_\mathcal{K} \subseteq adj(G,i,j) \cap path(G, i, j)$ with $|\mathcal{K}| = l$}
        \STATE $m = m + 1$
        \STATE Add $[i, j, \mathcal{K}]$ to $B[m, :]$
            \IF{m==b}
            \STATE Parallel Computation of correlation coefficients: Use batch $B$ and $corr$ to construct matrix $\mathbf{H}_0, \mathbf{H}_1$ and $\mathbf{H}_2$;
            \STATE Parallel CI-tests: Compare the estimated partial correlation coefficients $\boldsymbol{\rho}$ with the chosen threshold, and save the edges to be deleted;
            \STATE Delete these edges from the graph $G$.
            \STATE Reset $m = 0$ and $B$ to empty;
            \ENDIF
        \ENDFOR
        \ENDIF
    \ENDFOR
\ENDFOR

\IF{m != 0}
\STATE Execute line 10-13 once.
\ENDIF
\end{algorithmic}
\end{algorithm}


\section{Experiments}
In this section, we perform simulations to compare the efficiency, accuracy, and statistical power of the proposed algorithm with the PC algorithm (Algorithm 1), and we denote the algorithms as:
\begin{itemize}
    \item PC: Algorithm 1
    \item PC-reverse: Algorithm 2
    \item PC-reverse-parallel: Algorithm 2 + Parallelization 
    \item PC-reverse-parallel-gpu:  Algorithm 2 + Parallelization + GPU
\end{itemize}

\subsection{Data Generation}
We simulate the data to evaluate the algorithm's performance. First, we fix an order of a number of $n$ variables, and we draw data for the first random variable $X_1$ from a standard normal distribution 
\begin{equation*}
    X_1 = \epsilon_1 \sim \mathcal{N}(0,1).
\end{equation*}
The data for the other variables are generated using the following model
\begin{align*}
    & \epsilon_i \sim \mathcal{N}(0,1), \\
    & X_j = \sum_{i=1}^{j-1}s_{ij} \cdot a_{ij}X_i + \epsilon_j, j = 2,\ldots, n,
\end{align*}
where all $s_{ij}$ are random variables with Bernoulli distribution, $s_{ij} \sim Bernoulli(d)$ with $0 \leq d \leq 1$. The parameter $d$ decides the density of the graph.  $a_{ij}$ is a random variable uniformly distributed in the domain $[-0.8, -0.2] \cup [0.2, 0.8]$, that is, $c_N = 0.2, M = 0.8$, and all $\epsilon_i$'s are independent. If $s_{ij}=0$, there is no edge $X_i \rightarrow X_j$ in the true DAG.

We simulate Bayesian causal graphs with different sizes to get a benchmark of the proposed algorithm's speed compared to the conventional PC algorithm. In the true causal graph, we set the average number of neighbours to be 5. We also randomly choose two nodes to be the hub nodes and each has 10 neighbours. The significance level in all experiments are set to be $10^{-3}$. We simulate a variety of different size graphs with size $n$ varying from 15 to 100, increasing by 5. For each graph size, we generate 100 Monte Carlo samples graphs and each algorithm is run on all of them, and the final performance result is the average over 100 Monte Carlo samples.

\subsection{Running Time}
The result in Figure \ref{fig:runningtime} compares the running time of different algorithms. For the 95-node graph, the running time of PC, PC-reverse, PC-reverse-parallel, PC-reverse-parallel-gpu are 4237 seconds, 723 seconds, 89 seconds, 5.13 seconds, respectively. The later three have  6, 47, and 825 fold speed up compared to the PC algorithm. 

\begin{figure}[tb]
\centering
\includegraphics[width=12cm]{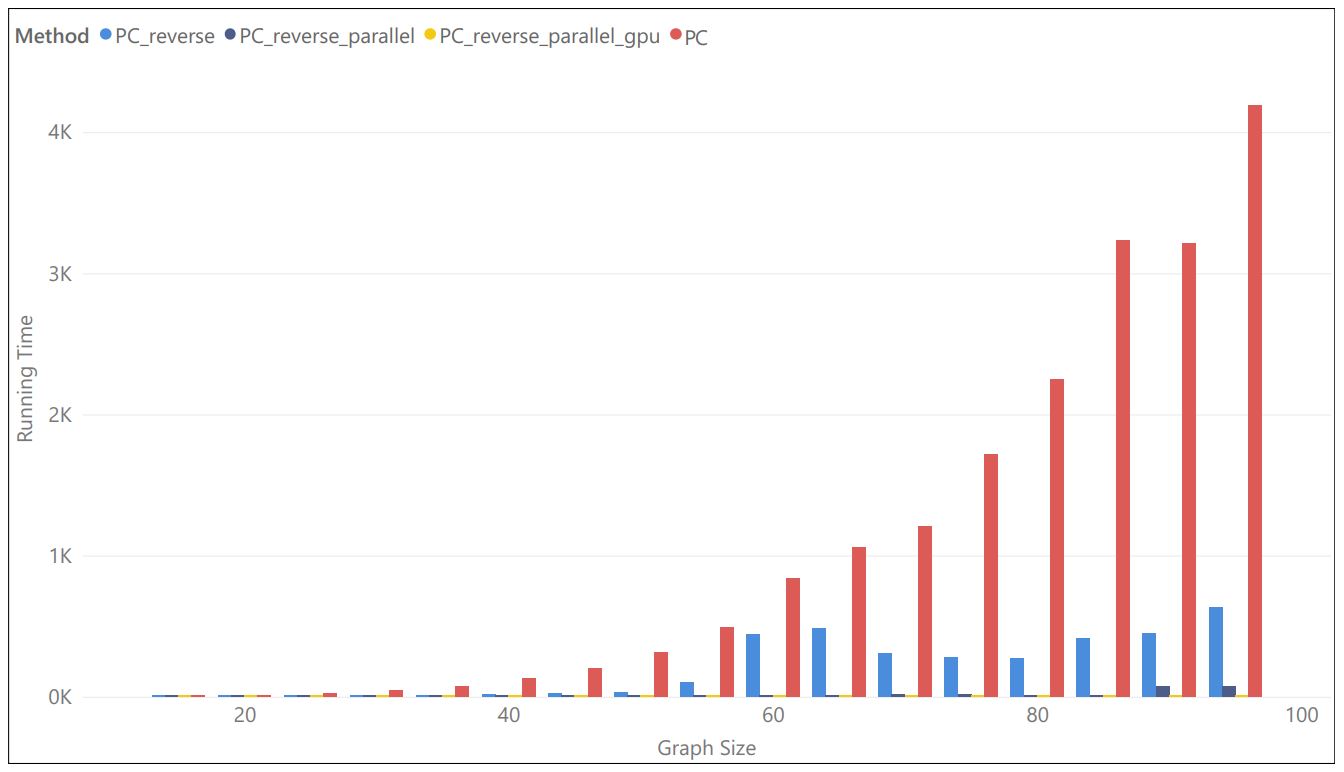}
\caption{Running time of different algorithms.}
\label{fig:runningtime}
\end{figure}

The running time of the algorithm is directly decided by the total number of CI tests performed. In Figure \ref{fig:noCI}, we compare the number of CI tests performed in different algorithms on different graphs. For the 95-node graph, PC algorithm performs 90394 tests, the PC-reverse and PC-reverse-parallel (PC-reverse-parallel-gpu) perform 8678 CI tests and 4563 CI tests, respectively. As the graph becoming denser, the improvement could be larger due to the number of tests has slower growth for the proposed algorithm. The PC-reverse-parallel-GPU and PC-reverse-parallel perform same on number of tests, because the only difference between them is computation platform (CPU or GPU).

\begin{figure}[tb]
\centering
\includegraphics[width=12cm]{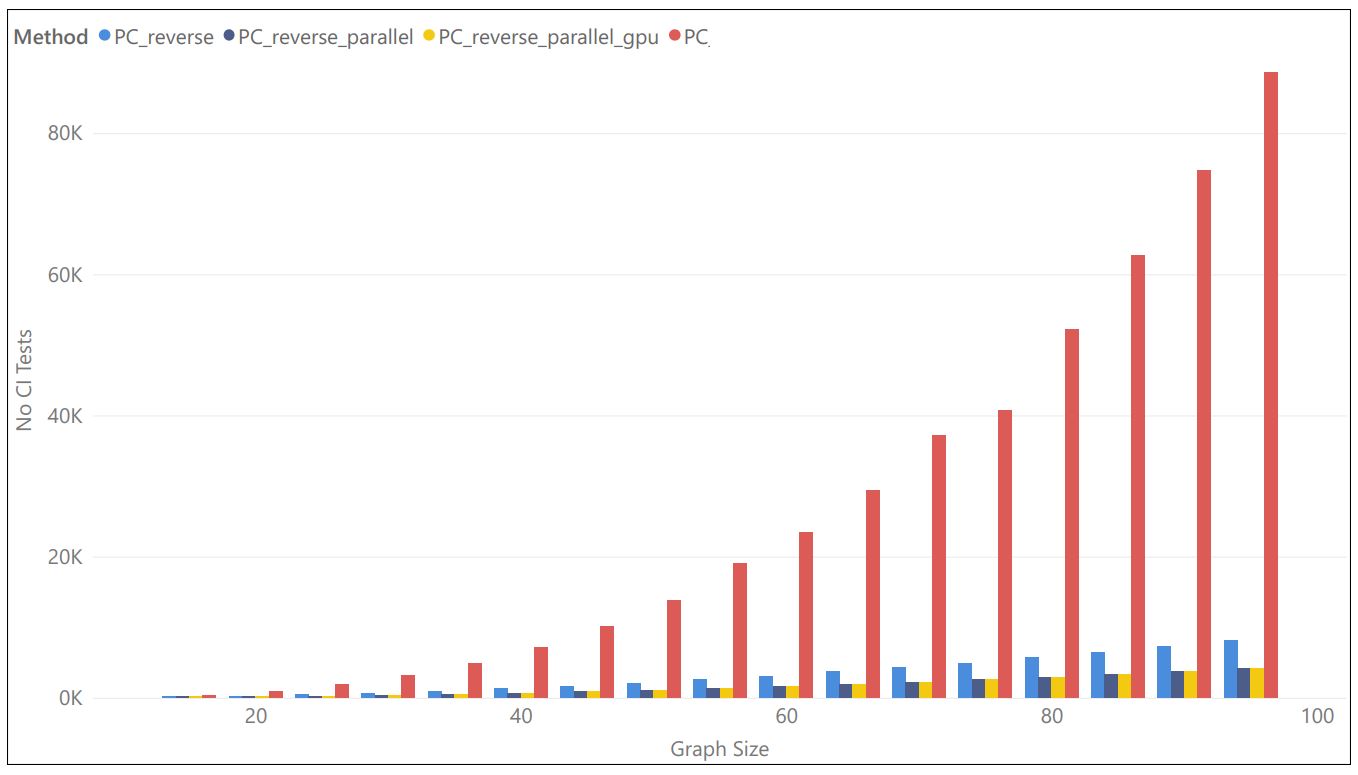}
\caption{Number of CI tests of different algorithms}
\label{fig:noCI}
\end{figure}

\subsection{Accuracy}
To compare the accuracy of the proposed algorithm and the conventional PC algorithm, we perform experiments on graphs of different sizes and evaluate their performances in terms of: false positive rate (FPR, Figure \ref{fig:fpr}), true positive rate (TPR, Figure \ref{fig:tpr}), and structural Hamming distance (SHD, Figure \ref{fig:shd}), respectively. All four algorithms have zero false positive rates, therefore the figure is omitted. All the above metrics are evaluated on the directed acyclic graph. The proposed algorithms performs slightly better compared to conventional PC algorithm.

\begin{figure}[tb]
\centering
\includegraphics[width=12cm]{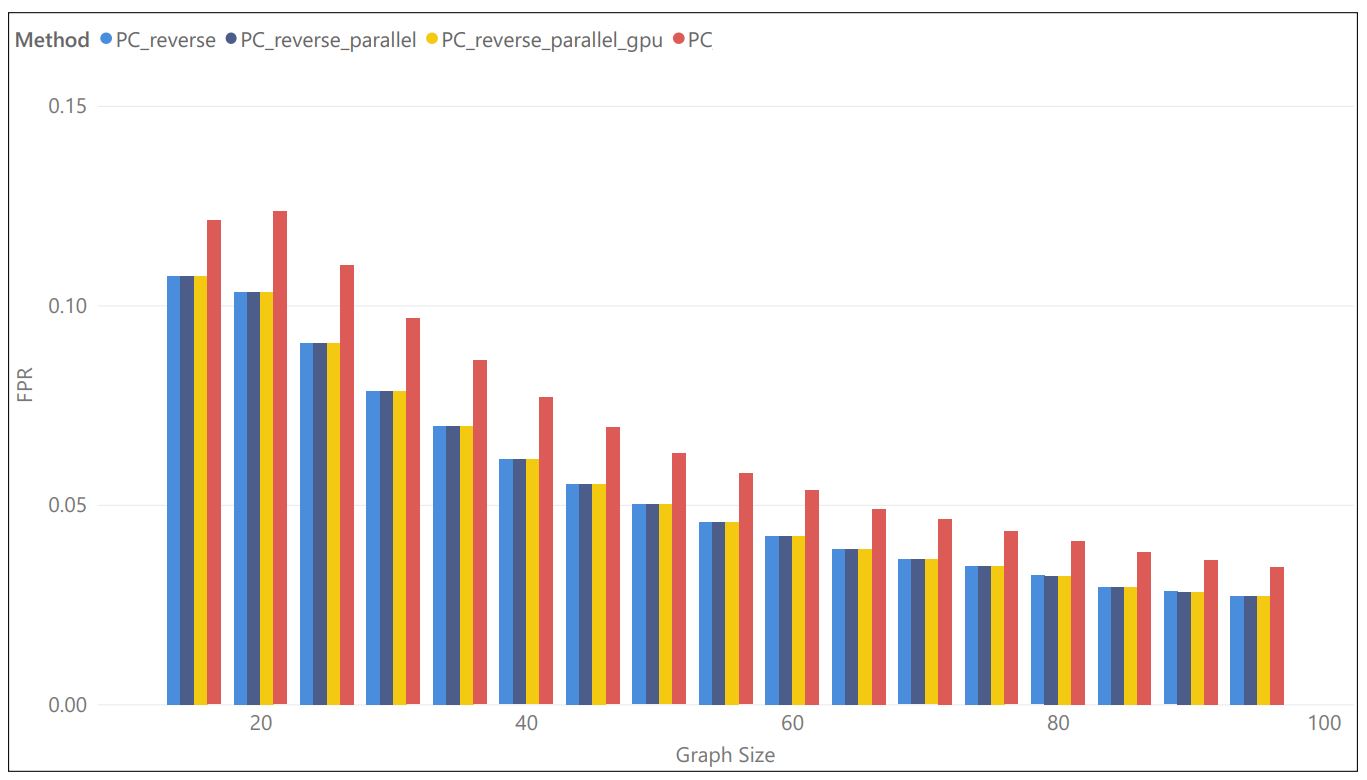}
\caption{False positive rate (FPR) of different algorithms.}
\label{fig:fpr}
\end{figure}

\begin{figure}[tb]
\centering
\includegraphics[width=12cm]{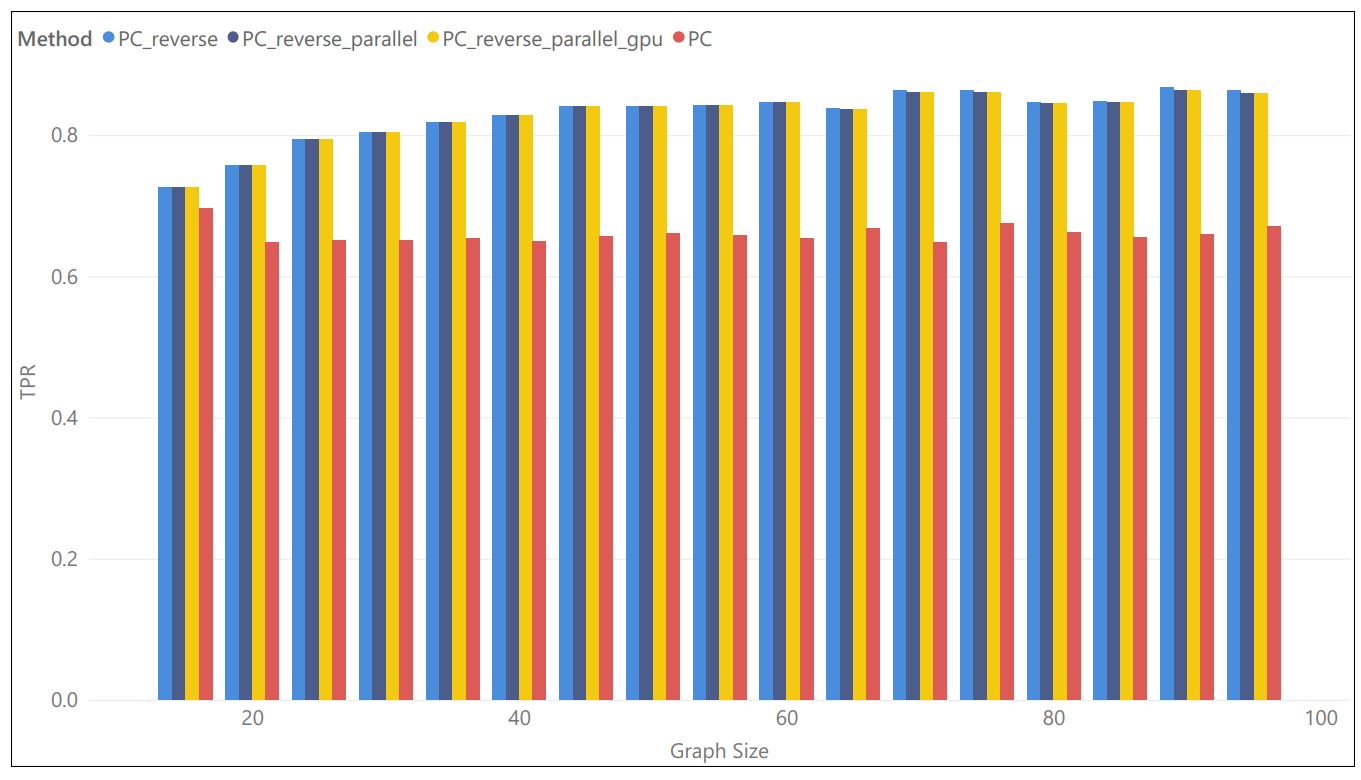}
\caption{True positive rate (TPR) of different algorithms.}
\label{fig:tpr}
\end{figure}

\begin{figure}[tb]
\centering
\includegraphics[width=12cm]{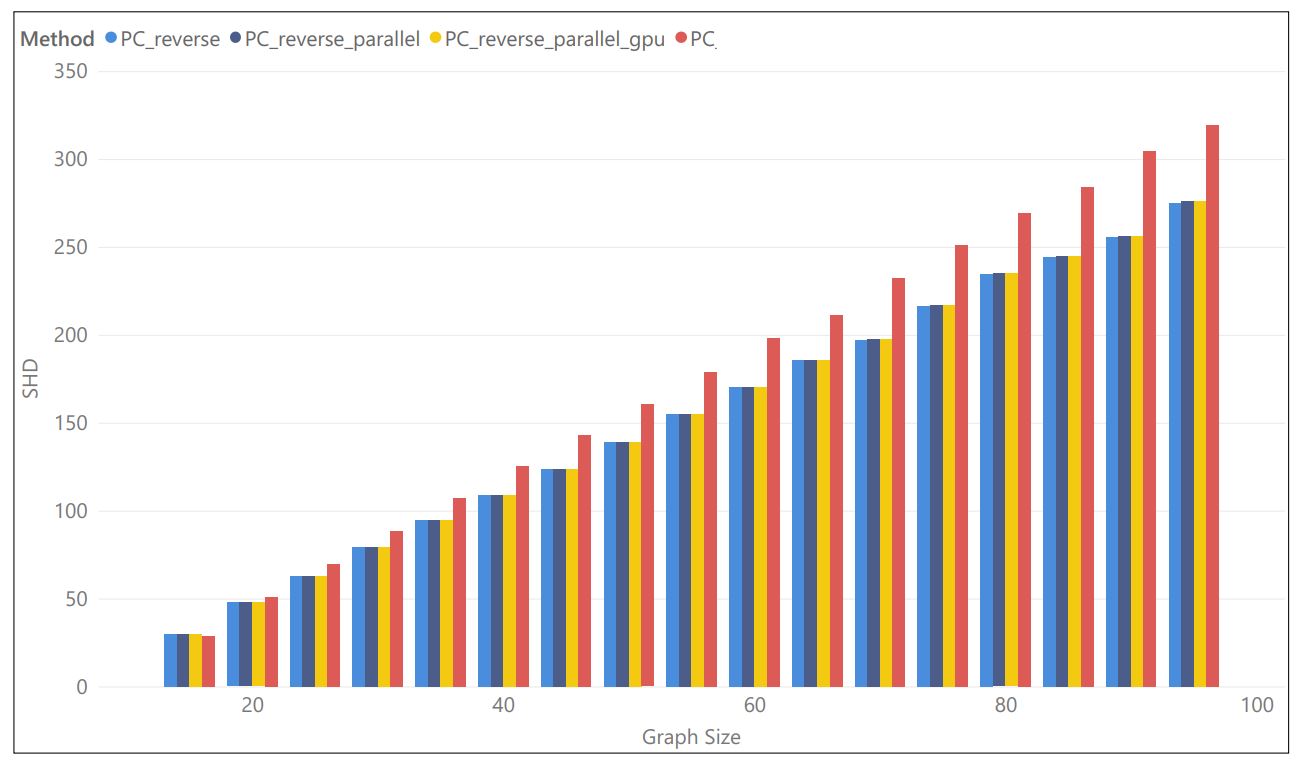}
\caption{Structural Hamming Distance (SHD) of different algorithms}
\label{fig:shd}
\end{figure}

\subsection{Potential Statistical Power Loss}
We generate a 10-node graph using the above data generation procedure. The true causal graph is shown in Fig. \ref{fig:true_graph_10_nodes}. From the true causal graph, we draw 500 Monte Carlo instances of dataset, each dataset contains the 10 random variables and $N=100,000$ data samples. We run the two algorithms 500 times on each of the dataset by setting the significance level to be $\alpha=10^{-3}$. We run the conventional PC algorithm as well as the proposed reverse pruning PC algorithm. Particularly, we focus on two edges, $X_5-X_6$ (independent) and $X_8-X_9$ (not independent). 

Obviously, $X_5$ and $X_6$ are independent given conditional set $\{X_0\}$, that is, $X_5 \CI X_6 \lvert \{X_0\}$. The PC algorithm will delete this edge at stage $l=1$. On the other hand, notice $X_5 \CI X_6 \lvert$ $\{X_0,X_1,X_2,X_3,X_4,X_7,X_8\}$ is also true, so PC-reverse algorithm should delete this edge at the stage $l=7$, which indeed happens in all 500 running. The random variables $X_8$ and  $X_9$ are dependent and the edge  $X_8-X_9$ are always kept in both algorithm. But we want to highlight two relations: $X_8 \nCI X_9 \lvert \{X_0\}$ and $X_8 \nCI X_9 \lvert \{X_0, X_1, X_2, X_3, X_4, X_5 ,X_6, X_7\}$.

\begin{figure}[tb]
\centering
\includegraphics[width=10cm]{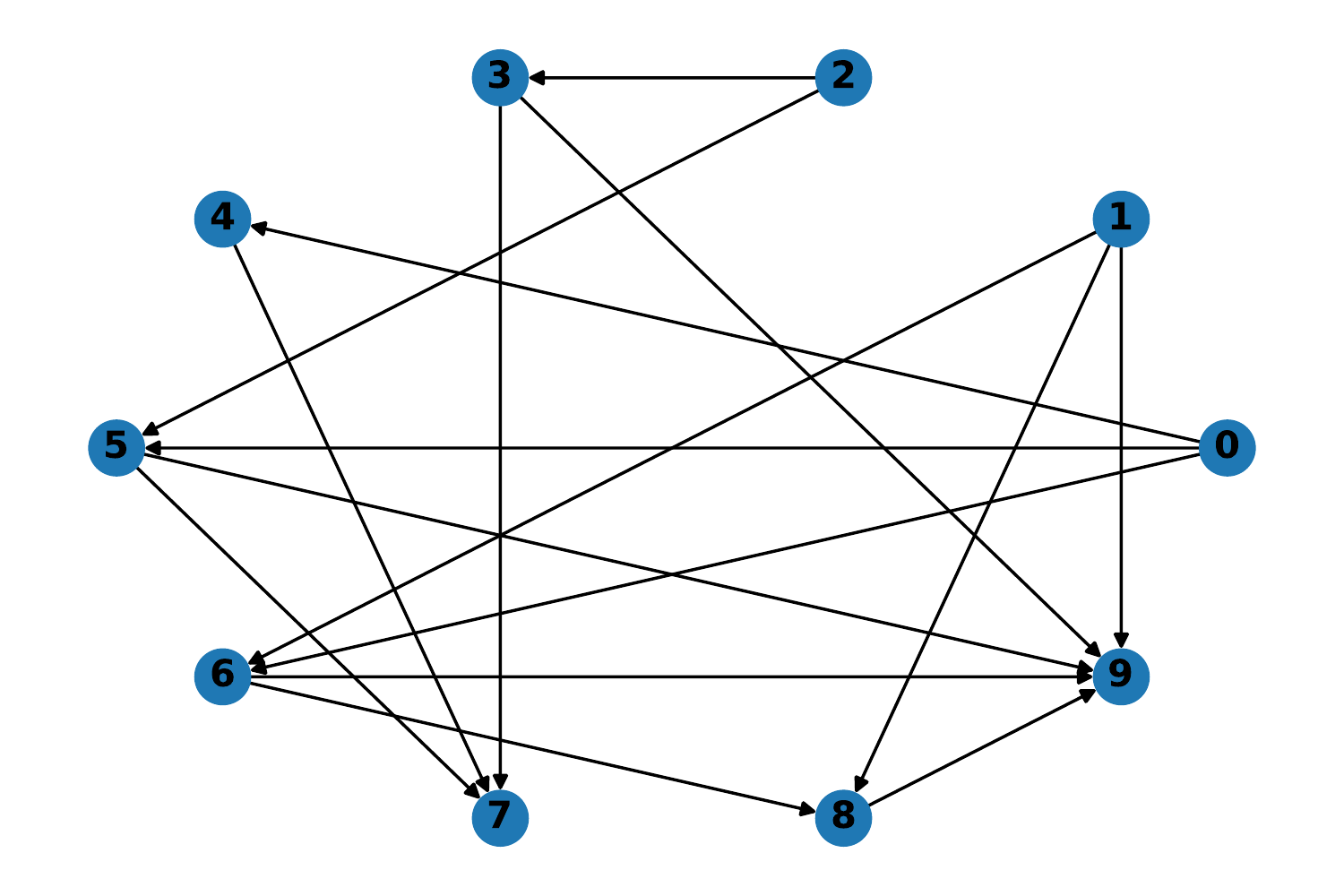}
\caption{An example causal graph of 10 nodes.}
\label{fig:true_graph_10_nodes}
\end{figure}

\subsubsection{Edge \texorpdfstring{$X_5 -X_6$}{lg}}
We plot each of the partial correlation coefficients $F_Z(r_{5,6 \cdot \{0\}})$ in all 500 running of the PC algorithm as DarkBlue in Fig. \ref{fig:Edge-5-6-forward}, and the 500 of $F_Z(r_{5,6 \cdot \{0,1,2,3,4,7,8\}})$ in the PC-reverse algorithm as LightBlue. The thresholds for these two CI tests are almost the same, $\frac{\Phi^{-1}(1-0.001/2)}{\sqrt{100000 - 1 - 3}} \approx \frac{\Phi^{-1}(1-0.001/2)}{\sqrt{100000 - 3}} = 0.01040577$, and is shown as the dashed line. It can be seen that in all 500 running, PC deletes the edge $X_5 -X_6$ successfully at stage $l=1$ because $F_Z(r_{5,6 \cdot \{0\}}) < \text{threshold}$, similarly, PC-reverse deletes the edge at stage $l=7$ in all 500 running, for the same reason.

\begin{figure}[tb]
\centering
\includegraphics[width=10cm]{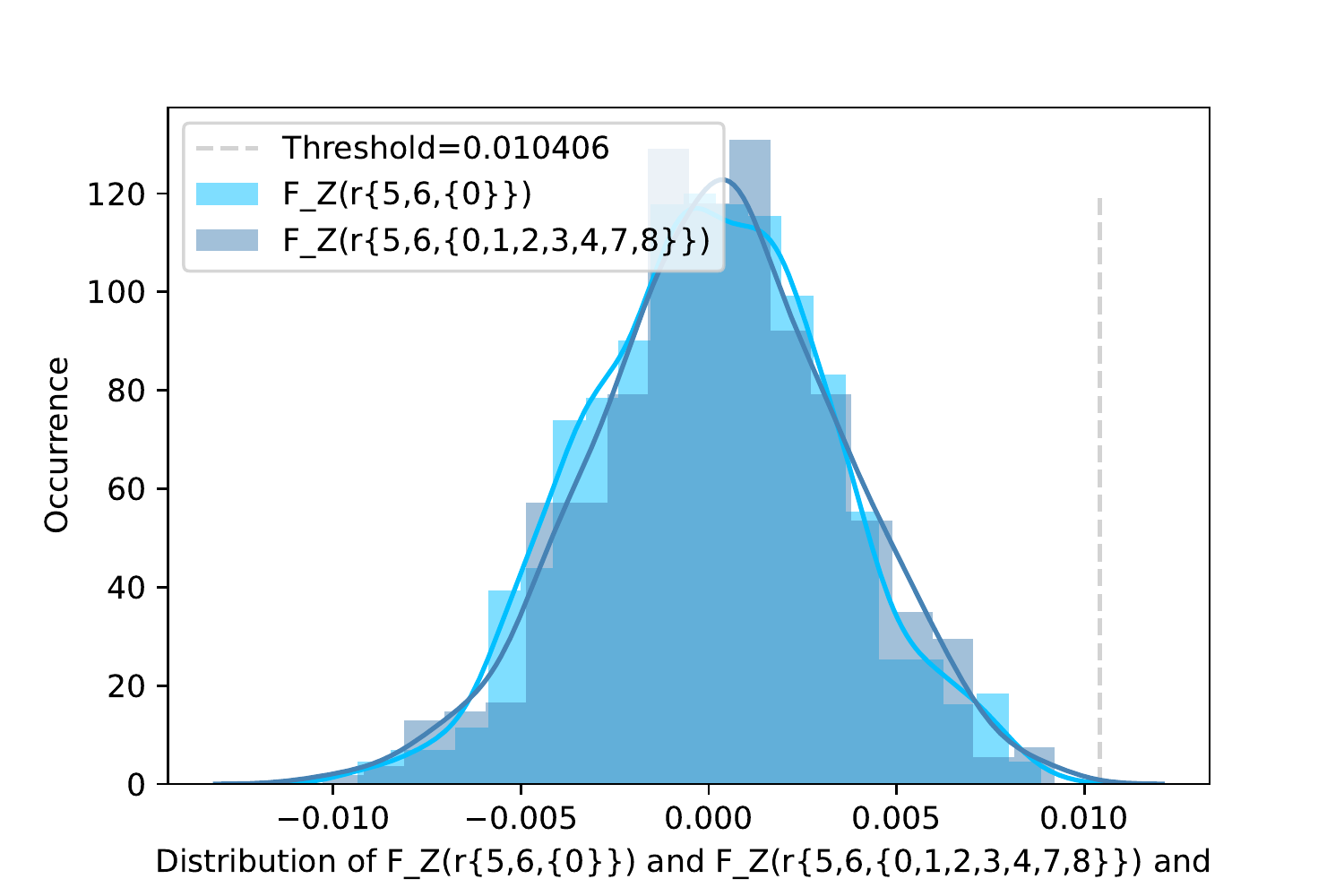}
\caption{The distribution of $F_Z(r_{5,6 \cdot \{0\}})$ and $F_Z(r_{5,6 \cdot \{0,1,2,3,4,7,8\}})$ after running for 500 times. The dotted line denotes the thresholds of 0.01040577.}
\label{fig:Edge-5-6-forward}
\end{figure}

\subsubsection{Edge \texorpdfstring{$X_8 -X_9$}{lg}}
We plot the 500 partial correlation coefficients $F_Z(r_{8,9 \cdot \{0\}})$ in the PC algorithm and the $F_Z(r_{8,9 \cdot \{0,1,2,3,4,5,6,7\}})$ in the PC-reverse algorithm in Fig. \ref{fig:Edge-8-9-forward}, shown as LightRed and DarkRed, respectively; and also the two thresholds for these two partial correlation coefficients are almost equal $\frac{\Phi^{-1}(1-0.001/2)}{\sqrt{100000 - 1 - 3}} \approx \frac{\Phi^{-1}(1-0.001/2)}{\sqrt{100000 - 8- 3}} = 0.01040577$.  Obviously, edge $X_8 - X_9$ is kept in all 500 running in both algorithms. 

The above settings are according to the assumptions (C1)-(C4). We proved both theoretically and experimentally, that add more variables into the conditional set will not affect the statistical power for originally independent relationships (such as edge $X_5-X_6$). However, for originally non-independent case (such as edge $X_8-X_9$), adding more add more variables into the conditional set could potentially reduce statistical power and wrongly delete an edge (due to estimation move closer to the threshold). This does not happen under our assumption (C1)-(C4) setting, but in real-world dataset it will also depend on other factors such as data noise.
 
\begin{figure}[tb]
\centering
\includegraphics[width=10cm]{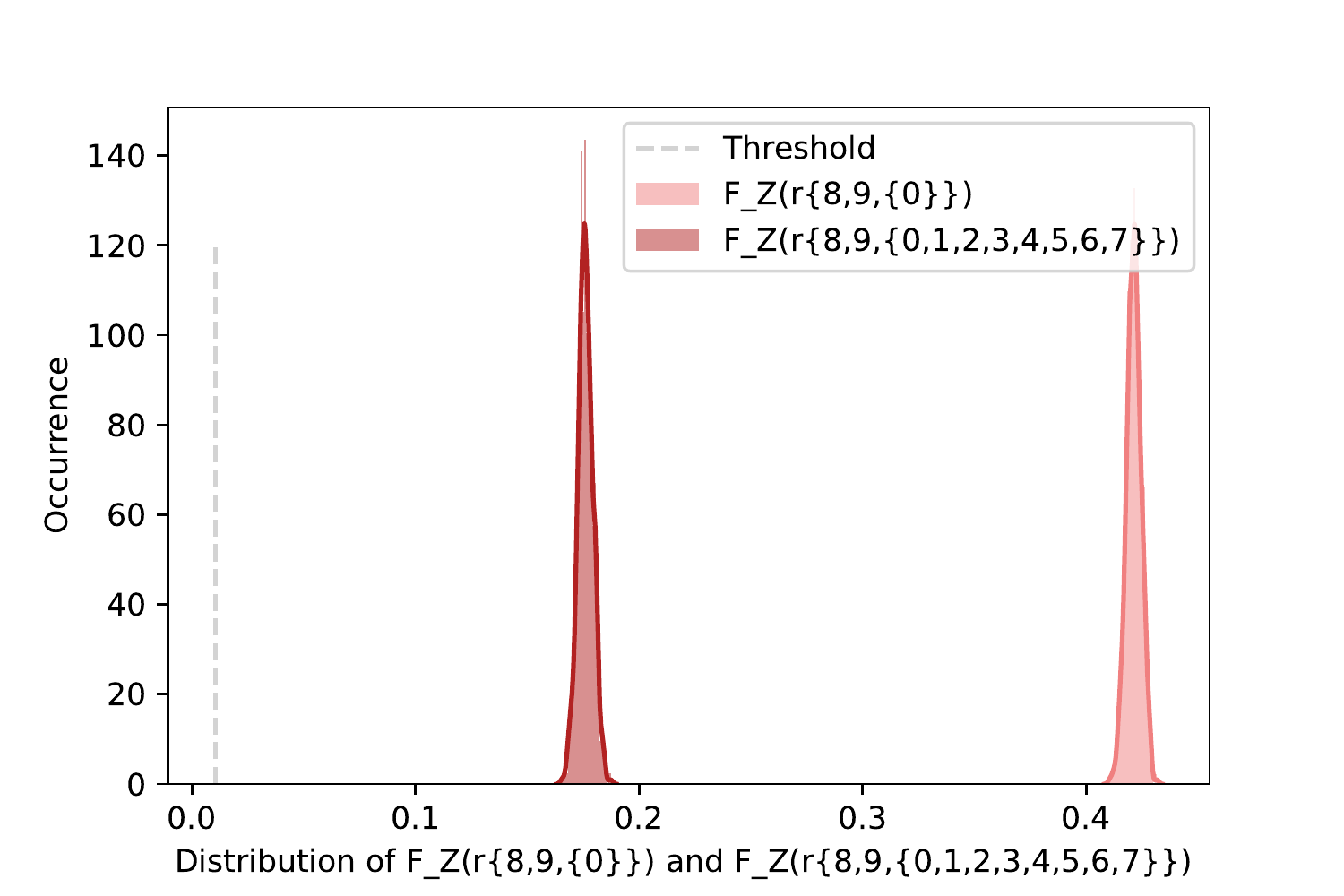}
\caption{The distribution of $F_Z(r_{8,9 \cdot \{0\}})$ and $F_Z(r_{8,9 \cdot \{0,1,2,3,4,5,6,7\}})$ after running for 500 times, the threshold is 0.01040577. }
\label{fig:Edge-8-9-forward}
\end{figure}

\subsection{On Real World Data}
In this section, we compare the conventional PC and the PC-reverse algorithms on real single-cell RNA sequencing data from Alzheimer Disease (AD) patients. The data is from a total of 48 patients (24 AD/ 24 non-AD) single cell data from the Religious Orders Study and Memory and Aging Project (ROSMAP). We use the 171 already-known AD-related genes as random variables. The sample (cell) number is $N=70,634$, including eight cell types. To achieve comparable sparsity grade, we set the significance level to be $10^{-4}$ for the PC algorithm and $10^{-6}$ for the PC-reverse algorithm. The PC-reverse demonstrates a 21.6-fold speed up compared to PC (21.6 seconds vs. 455 seconds). The PC and PC-reverse identifies 434 and 439 directed edges, respectively, and 196 of the edges are identified by both algorithms. If we only focus on the skeleton (undirected edges), then there are 296 mutual edges.

\begin{figure}[tb]
\centering
\includegraphics[width=12cm]{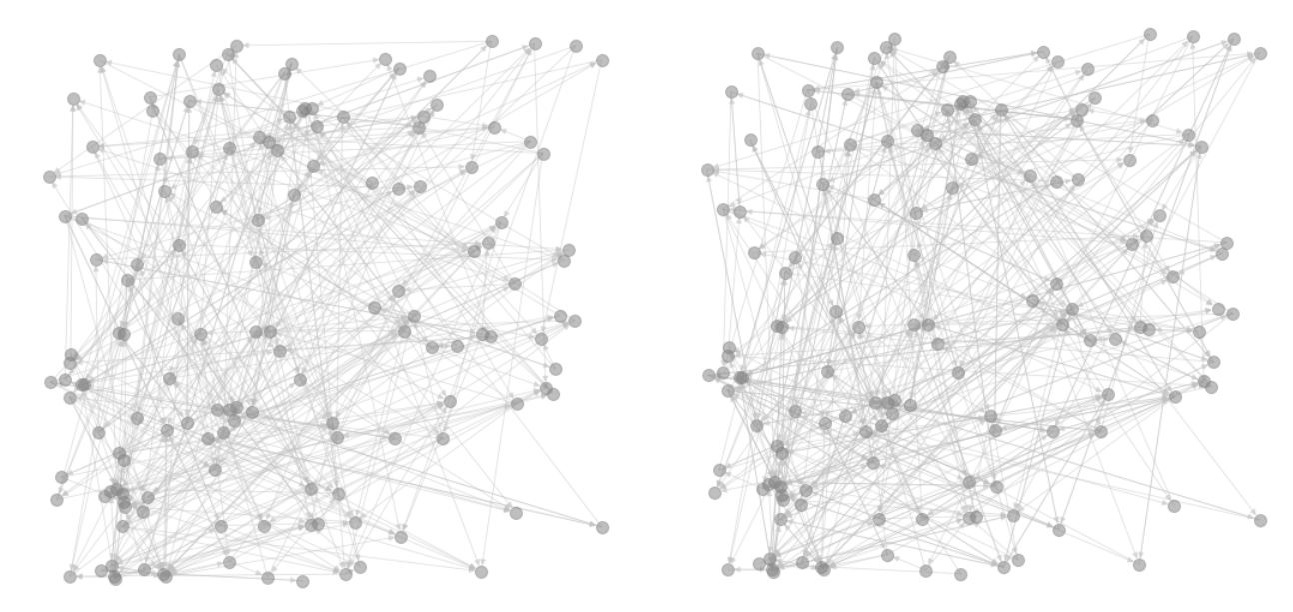}
\caption{Causal structure learned by PC and PC-reverse on Single-cell RNA sequencing (scRNA-seq) data. Left: Conventional PC; Right: PC-reverse.}
\label{fig:Edge-8-9-backward}
\end{figure}

\section{Conclusion}
We propose a fast causal discovery algorithm based on the PC algorithm. The proposed algorithm demonstrates up to near-thousandfold speed up on a simulated 100-node graph. We prove the consistency of our algorithm's result on real dataset. We also analyzed the statistical power of the proposed algorithm which is of no loss asymptotically under the mild assumptions of the data and graph dimension. Based on simulation results, our algorithm even achieves slightly higher accuracy in terms of TPR, FPR and SHD. We also provide a parallel version of the proposed algorithm and it can be GPU-accelerated. They can achieve significant speed-up compared to the conventional PC algorithm. The proposed algorithm is evaluated on a real-world dataset and demonstrates significant speed-up compared to the conventional PC algorithm.

\appendix


\section{Proof of Proposition \ref{prop:prop1}}\label{appendix:prop:prop1}
\begin{proof}
The proof is rather simple by using the rules of $d$-separation. 

The set $\mathcal{K}_{ij}^{\min}$ indicates the minimal set of variables that renders $X_i$, $X_j$ independent, which blocks every information flow path between them. Since $\mathcal{K}_{ij}^{\max}$ includes $\mathcal{K}_{ij}^{\min}$, we are ensured that all such paths are still blocked, hence Step 1. Next, we show that conditioning on the additional variables in $\mathcal{K}_{ij}^{\max}$ does not create any other paths between $X_i$ and $X_j$ (the Step 2). Indeed, conditions (R1) and (R2) are for this objective.

\begin{figure}[!htb]
\minipage{0.33\textwidth}
  \includegraphics[width=0.7\linewidth]{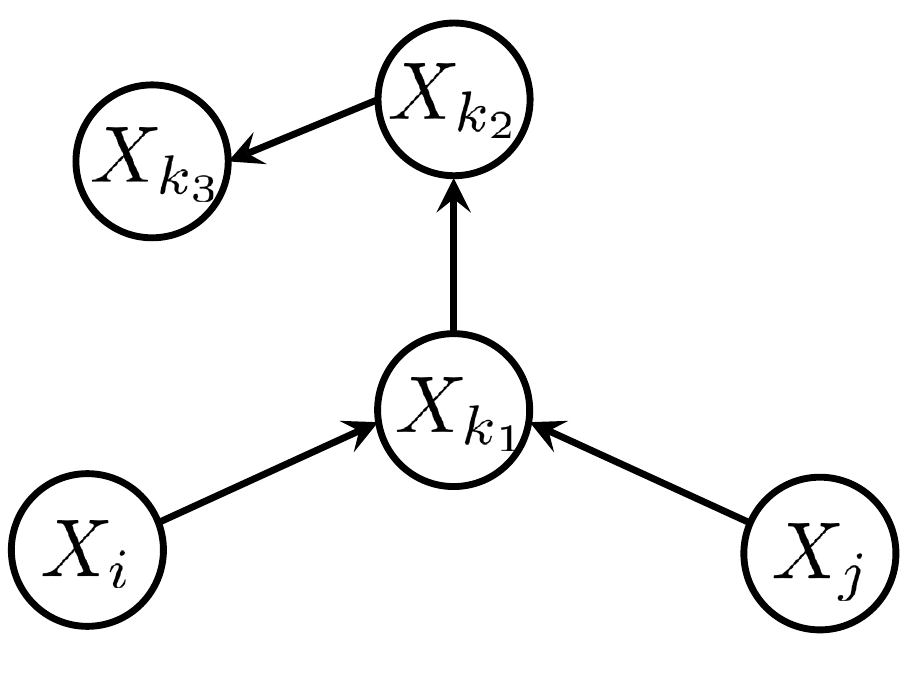}
  \caption{Example 1.}\label{fig:prop1-figure1}
\endminipage
\minipage{0.33\textwidth}%
  \includegraphics[width=0.7\linewidth]{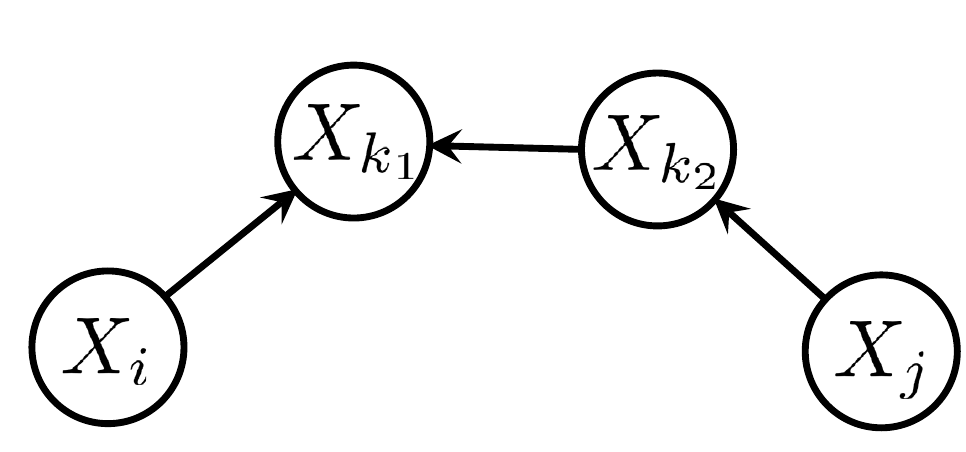}
  \caption{Example 2.} \label{fig:prop1-figure3}
\endminipage
\minipage{0.3\textwidth}%
  \includegraphics[width=1\linewidth]{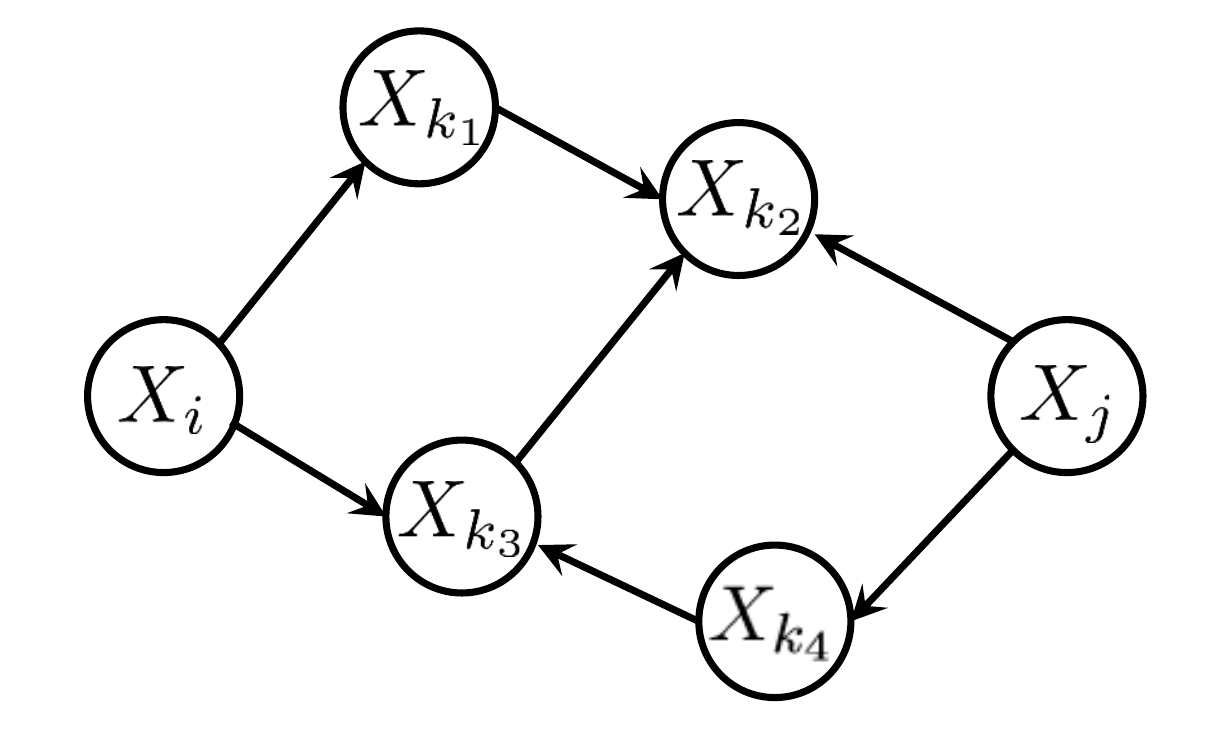}
  \caption{Example 3.} \label{fig:prop1-figure5}
\endminipage
\end{figure}

The two rules follow the same spirit of the fact that: If a collider is a member of the conditioning set Z, or has a descendant in Z, then it no longer blocks any path that traces this collider \cite{judea2000causality}. Therefore, Step 2 is to inflate the conditional set carefully without opening a new path between the two nodes.

Rule 1 is because, conditioning on a $v$-structure node or any of its descendants will open a path between $X_i$ and $X_j$ thus making them dependent, see Figure \ref{fig:prop1-figure1}. Therefore, any $v$-structure node and its descendants should be excluded in $\mathcal{K}_{ij}^{\max}$, that is, $X_{\mathcal{W}_{ij}} \cup de(X_{\mathcal{W}_{ij}})$.

Rule 2 is to exclude other types of colliders from the conditional set, which otherwise will also open a path. For example, the $X_{k_1}$ in Figure \ref{fig:prop1-figure3}. But if $X_{k_1}$ is included, we should also include some other non-collider nodes on this path ($X_{k_2}$), which will block this path again. 

A slightly more complicated example is given in Figure \ref{fig:prop1-figure5}. Originally, $\mathcal{K}_{ij}^{\min} = \emptyset$ since $X_i \CI X_j \; \vert \; \emptyset$. To find the largest conditional set $\mathcal{K}_{ij}^{\max}$, Rule 1 states we can not include $X_{k_2}, X_{k_3}$. The only options left for us are $X_{k_1}$ and $X_{k_4}$. Based on Rule 3, if we include $X_{k_3}$, at least another non-collider variable on the path where $X_{k_3}$ is a collider should also be included, which is $X_{k_4}$, indeed, $X_i \CI X_j \; \vert \; \{X_{k_3}, X_{k_4}\}$. Similarly, if we include $X_{k_1}$, then $X_{k_2}$ can also be included. Therefore $\mathcal{K}_{ij}^{\max} = \{X_{k_1}, X_{k_2}, X_{k_3}, X_{k_4}\}$.
\end{proof}

\subsection{Proof of Corollary \ref{coro:timecomplexity}}\label{appendix:coro:timecomplexity}
\begin{proof}
The PC (and PC-stable) algorithm proceeds by increasing the order of conditional independence queries or tests. Assume that for an edge $X_i - X_j$, $X_i \CI X_j \vert X_{\mathcal{K}}$ for some $X_{\mathcal{K}}$, the algorithm will find the minimal set of such $\mathcal{K}$ when $\mathcal{K} = \mathcal{K}_{ij}^{\min}$ at stage $l = |\mathcal{K}_{ij}^{\min}|$. Since the algorithm starts with a fully connected graph, at stage $l=0$ it checks the empty set $\mathcal{K} = \emptyset$; at stage $l=1$ it checks all singleton set of all the other $n-2$ neighbors, and so forth, therefore,
\begin{equation*}
    Q_{ij}(n) = \sum_{i=0}^{|\mathcal{K}_{ij}^{\min}|}\binom{n-2}{i} \approx (n-1)^{|\mathcal{K}_{ij}^{\min}|}.
\end{equation*}

The proposed reverse order pruning PC algorithm starts with a fully connected graph and proceeds by decreasing the order of conditional independence queries or CI tests. At stage $l=n-2$ it checks the set of all other nodes $X_{\mathcal{K}} = X_{\mathcal{V}} \backslash \{X_i, X_j\}$; at stage $l=n-3$ it checks all set of $n-3$ nodes from all the rest nodes, and so forth. 

Based on \ref{prop:prop1}, when set $\mathcal{K}$ excludes all nodes of $X_{\mathcal{V}} \backslash (X_i \cup X_j \cup X_{\mathcal{K}_{ij}^{\max}})$, the edge $X_i - X_j$ will be deleted.

Therefore,
\begin{align*}
    Q_{ij}^r(n) = \sum_{i=|\mathcal{K}_{ij}^{\max}|}^{n-2}\binom{n-2}{i} = \sum_{i=0}^{n-2-|\mathcal{K}_{ij}^{\max}|}\binom{n-2}{i} \approx (n-1)^{n-2-|\mathcal{K}_{ij}^{\max}|},
\end{align*}
and 
\begin{equation*}
    \frac{Q_{ij}(n)}{Q_{ij}^r(n)} = (n-1)^{|\mathcal{K}_{ij}^{\min}| + |\mathcal{K}_{ij}^{\max}| - n + 2}.
\end{equation*}
\end{proof}

\subsection{Proof of Proposition \ref{prop:whichstage}}\label{appendix:prop:whichstage}
\begin{proof}
If $X_i$ and $X_j$ are not adjacent in the true DAG, there must exist a minimal set $\mathcal{K}_{ij}^{\min} \subseteq \mathcal{V} \backslash \{i,j \}$ such that $X_i \CI X_j \vert X_{\mathcal{K}_{ij}^{\min}}$. Base on Proposition \ref{prop:prop1}, we can find a maximal $\mathcal{K}_{ij}^{\max}$ which is a superset of $X_{\mathcal{K}^{\min}}$ such that $X_i \CI X_j \vert X_{\mathcal{K}_{ij}^{\max}}$.

Since the algorithm iterates all possible $\mathcal{K}$ with its size decreasing, proposition \ref{prop:prop1} states that when the conditional set $\mathcal{K}$ is exactly $\mathcal{K}_{ij}^{\max}$, the edge will be deleted, which is at stage $l = |\mathcal{K}_{ij}^{\max}|$.

If $X_i$ and $X_j$ are adjacent in the true DAG, then there is no such $\mathcal{K}_{ij}$ that renders them independent, in this case, the algorithm will proceed until reaches the end, which is stage $l = 0$. 
\end{proof}

\subsection{Proof of Theorem \ref{them:skeletonconsistency}}\label{appendix:them:skeletonconsistency}
The proof is enlightened by the proof of a similar theorem in \cite{kalisch2007estimating}. The Lemma \ref{lemma:correlationBound}, \ref{lemma:partialCorrelationBound} and \ref{lemma:zBound} take the same spirit as Lemma 1-3 in \cite{kalisch2007estimating} but with some modifications, and the probability is bounded differently. For completeness, we still provide a complete proof for each of the following Lemma \ref{lemma:correlationBound}, \ref{lemma:partialCorrelationBound}, \ref{lemma:zBound}. They are useful to prove our theorem. 

In the following, we use $r$ to denote an estimation of the (partial) correlation coefficient $\rho$. Throughout the entire reversed order pruning PC algorithm, an edge $X_i - X_j$ may be tested multiple times on different conditional set $\mathcal{K}$, and here we define a superset to incorporate all such $\mathcal{K}$, which will be useful for our proof.

Based on Proposition \ref{prop:whichstage}, we define two different sets regarding whether an edge $X_i - X_j$ exists in the true DAG or not.

\begin{itemize}
    \item If an edge $X_i - X_j$ does not exist in the true DAG, define
\begin{equation*}
    \mathcal{T}_{i,j}^{m} \triangleq \{\mathcal{K} \subseteq \{1, \ldots, n_N\} \backslash \{i, j\}: \mathcal{K} \cap \mathcal{K}^{\text{max}}_{ij} \neq \emptyset, \; |\mathcal{K}| \geq  m\},
\end{equation*}
where $m = |\mathcal{K}^{\text{max}}_{ij}|$. The tuple $(i, j, \mathcal{T}_{i,j}^{m} )$ denotes the two nodes of a non-existing edge $X_i-X_j$ in the true DAG together with all possible $\mathcal{K}$ when testing this edge. 
    \item If an edge $X_i - X_j$ exists in the true DAG, define
\begin{equation*}
    \mathcal{T}_{i,j} \triangleq \{\mathcal{K} \subseteq \{1, \ldots, n_N\} \backslash \{i, j\}\}.
\end{equation*}
Similarly, the tuple $(i, j, \mathcal{T}_{i,j})$ denotes the two nodes of an existing edge $X_i-X_j$ in the true DAG together with all possible $\mathcal{K}$ when testing this edge.
\end{itemize}

\begin{lemma}[Lemma 1 in \cite{kalisch2007estimating}]\label{lemma:correlationBound}
Suppose the probability distribution of the $n_N$ random variables are multi-variant Gaussian, and $\sup_{i, j \in \{1,\ldots, n_N\}}|\rho_{ij}| \leq M < 1$, then 
\begin{equation*}
    \sup_{(i,j, \mathcal{T}_{i,j}^m) \cup (i, j, \mathcal{T}_{i,j})} P(|r_{ij} - \rho_{ij}| > \gamma) \leq C_1(M) \cdot (N-2) \cdot \exp \Big( (N-4)\log \frac{4-\gamma^2}{4+\gamma^2} \Big),
\end{equation*}
for any $0 < \gamma \leq 2$, and constant $C_1$ which only depends on $M$. 

\end{lemma}

\begin{proof}
In the following, we omit the subscript $i,j$ in $r_{ij}$ and $\rho_{ij}$ when discussing the estimated and true correlation coefficient.

The work \cite{hotelling1953new} provides a detailed analysis of the distribution of $r$, the distribution density for $r$ is provided as $f_N(r,\rho)$ in \cite{hotelling1953new} eq. (25). It was also proved that the distribution function has the property of $f_N(r,\rho) = f_N(-r,-\rho)$.

The left hand side of the inequality to be proved can be expressed as, for any $i,j$ and $0 < \gamma \leq 2$,
\begin{align*}
    P(|r  - \rho | > \gamma) = & P(r  > \rho  + \gamma) + P(r   < \rho  - \gamma) \\
    = & P(r  > \rho  + \gamma) + P(- r   > - \rho  + \gamma) \\
    = & 2P(r  > \rho  + \gamma).
\end{align*}

Based on \cite{hotelling1953new} eq. (25), the above probability is upper-bounded by
\begin{align*}
    & P(r  > \rho  + \gamma) \\
    \leq & \frac{(N'-1)\Gamma(N')}{\sqrt{2\pi}\Gamma(N'+\frac{1}{2})} \Bigg( 1+\frac{2}{1-|\rho|} \Bigg)  \int_{\rho + \gamma}^{1} (1-\rho^2)^{\frac{N'}{2}}(1-x^2)^{\frac{N'-3}{2}}(1-\rho x)^{\frac{1}{2}-N'}dx
\end{align*}
where $N' = N-1$. Denote the above integral as $h(\rho, \gamma)$, then
\begin{align*}
    h(\rho, \gamma) = & \frac{(1-\rho^2)^{\frac{3}{2}}}{(1-\rho)^{\frac{5}{2}}} \int_{\rho + \gamma}^1 \Bigg( \frac{\sqrt{1-\rho^2}\sqrt{1-x^2}}{1-\rho x} \Bigg)^{N'-3}dx \\
    \leq & \frac{(1-(\rho + \gamma))(1-\rho^2)^{\frac{3}{2}}}{(1-\rho)^{\frac{5}{2}}} \max_{\rho + \gamma \leq x \leq 1} \Bigg( \frac{\sqrt{1-\rho^2}\sqrt{1-x^2}}{1-\rho x} \Bigg)^{N'-3} \\
    = & \frac{(1-(\rho + \gamma))(1-\rho^2)^{\frac{3}{2}}}{(1-\rho)^{\frac{5}{2}}} \Bigg( \frac{\sqrt{1-\frac{\gamma^2}{4}}\sqrt{1-\frac{\gamma^2}{4}}}{1-\frac{-\gamma}{2} \frac{\gamma}{2}} \Bigg)^{N'-3} \\
    = & \frac{(1-(\rho + \gamma))(1-\rho^2)^{\frac{3}{2}}}{(1-\rho)^{\frac{5}{2}}} \Bigg( \frac{4-\gamma^4}{4+\gamma^2} \Bigg)^{N'-3}.
\end{align*}

Therefore, 
\begin{align*}
    & P(|r - \rho| > \gamma) = 2P(r > \rho + \gamma) \\
    \leq & 2\frac{(N'-1)\Gamma(N')}{\sqrt{2\pi}\Gamma(N'+\frac{1}{2})} \Bigg( 1+\frac{2}{1-|\rho|} \Bigg)  \frac{(1-(\rho + \gamma))(1-\rho^2)^{\frac{3}{2}}}{(1-\rho)^{\frac{5}{2}}} \Bigg( \frac{4-\gamma^4}{4+\gamma^2} \Bigg)^{N'-3}.
\end{align*}
Using the fact that $|\rho| \leq M < 1$ from (C3) and $\frac{\Gamma(N')}{\Gamma(N'+1/2)} \leq const$, we have 
\begin{align*}
    P(|r  - \rho | > \gamma)  \leq & 2\frac{(N'-1)\Gamma(N')}{\sqrt{2\pi}\Gamma(N'+\frac{1}{2})} \Bigg( 1+\frac{2}{1-|\rho|} \Bigg)  \frac{(1-(\rho + \gamma))(1-\rho^2)^{\frac{3}{2}}}{(1-\rho)^{\frac{5}{2}}} \Bigg( \frac{4-\gamma^4}{4+\gamma^2} \Bigg)^{N'-3} \\
    \leq & 2\frac{(N'-1)\Gamma(N')}{\sqrt{2\pi}\Gamma(N'+\frac{1}{2})} \Bigg( 1+\frac{2}{1-|\rho|} \Bigg)  \frac{1}{(1-\rho)^{\frac{5}{2}}} \Bigg( \frac{4-\gamma^4}{4+\gamma^2} \Bigg)^{N'-3} 
\end{align*}   
\begin{align*}
    \leq & 2\frac{(N'-1)\Gamma(N')}{\sqrt{2\pi}\Gamma(N'+\frac{1}{2})} \Bigg( 1+\frac{2}{1-M} \Bigg)  \frac{1}{(1-M)^{\frac{5}{2}}} \Bigg( \frac{4-\gamma^4}{4+\gamma^2} \Bigg)^{N'-3} \\
    \leq & C_1(M) \cdot (N'-1) \Bigg( \frac{4-\gamma^4}{4+\gamma^2} \Bigg)^{N'-3} \\
    = & C_1(M) \cdot (N-2)\exp \big( (N-4)\log\frac{4-\gamma^4}{4+\gamma^2} \big)
\end{align*}
where $C_1(M) = \frac{2}{\sqrt{2\pi}} (1+\frac{2}{1-M})\frac{1}{(1-M)^{\frac{5}{2}}}$ is a function of only $M$, and $C_1(M) \in (2.39365, \infty)$ when $M \in (0, 1)$.
\end{proof}

\begin{lemma}[Lemma 2 in \cite{kalisch2007estimating}]\label{lemma:partialCorrelationBound}
Suppose the probability distribution of the $n_N$ random variables are multi-variant Gaussian, and $\sup_{i, j \in \{1,\ldots, n_N\}}|\rho_{ij}| \leq M < 1$. 
\begin{equation*}
    \sup_{(i,j, \mathcal{T}_{i,j}^m) \cup (i, j, \mathcal{T}_{i,j})} P(|r_{ij \cdot \mathcal{K}} - \rho_{ij \cdot \mathcal{K}}| > \gamma) \leq C_1(M) \cdot (N-n_N) \cdot \exp \Big( (N-n_N-2)\log \frac{4-\gamma^2}{4+\gamma^2} \Big),
\end{equation*}
for any $0 < \gamma \leq 2$, and constant $C_1(M)$, $2.39365 < C_1(M) < \infty$ as $M$ varies in $(0, 1)$.
\end{lemma}
\begin{proof}
The proof is trivial using the following fact. For $n_N$ random variables with multi-variant Gaussian distribution and any two variables $X_i$ and $X_j$, if the CDF of estimated correlation coefficient $r_{ij}$ is denoted as $F(\cdot \vert N, \rho_{ij})$, then the CDF of estimated partial correlation coefficient $r_{ij \cdot \mathcal{K}}$ is $F(\cdot \vert N - |\mathcal{K}|\vert , \rho_{ij \vert \mathcal{K}})$, see \cite{fisher1924distribution} page 330.
\begin{align}
    & \sup_{(i,j, \mathcal{T}_{i,j}^m)  \cup (i, j, \mathcal{T}_{i,j})} P(|r_{ij \cdot \mathcal{K}} - \rho_{ij \cdot \mathcal{K}}| > \gamma) \notag \\
    \leq & C_1(M) \cdot (N-2-|\mathcal{K}|) \cdot \exp \Big( (N-4-|\mathcal{K}|) \log \frac{4-\gamma^2}{4+\gamma^2} \Big) \label{eq:intermediate}.
\end{align}

The term $\log \frac{4-\gamma^2}{4+\gamma^2} < 0$ for $\gamma > 0$ and the function $f(x) = xe^{-(x-2)}$ is monotonically decreasing when $x > 1$. Therefore, formula \eqref{eq:intermediate} is upper-bounded by replacing $|\mathcal{K}|$ with its largest possible value. 

For an edge $X_i - X_j$ does not exist in the true DAG, the proposed algorithm iterates from stage $l=n_N-2$ to $n_N-2-m$ and deletes the edge when $\mathcal{K}$ is $\mathcal{K}^{\text{max}}_{ij}$. On the other hand, if edge $X_i - X_j$ does not exist in the true DAG, the proposed algorithm iterates from stage $l=n_N-2$ to 0. Therefore, $\max_{(i,j, \mathcal{K} \in \mathcal{T}_{i,j}^m)}|\mathcal{K}| = \max_{(i,j, \mathcal{K} \in \mathcal{T}_{i,j})}|\mathcal{K}| = n_N-2$.

Therefore we have
\begin{align*}
    \eqref{eq:intermediate} \leq & C_1(M) \cdot (N-2-(n_N-2)) \cdot \exp \Big( (N-4-(n_N-2)) \log \frac{4-\gamma^2}{4+\gamma^2} \Big) \\
    \leq & C_1(M) \cdot (N-n_N) \cdot \exp \Big( (N-n_N-2)\log \frac{4-\gamma^2}{4+\gamma^2} \Big),
\end{align*}
where the $C_1(M)$ is the same as defined in Lemma \ref{lemma:correlationBound}.
\end{proof}

\begin{lemma}[Lemma 3 in \cite{kalisch2007estimating}]\label{lemma:zBound}
Suppose the probability distribution of the $n_N$ random variables are multi-variant Gaussian, and $\sup_{i, j \in \{1,\ldots, n_N\}}|\rho_{ij}| \leq M < 1$, then for $\mathcal{K} \subseteq \{1,\ldots,n_N\} \backslash \{i,j\}$,
\begin{equation*}
    \sup_{(i,j, \mathcal{T}_{i,j}^m) \cup (i,j, \mathcal{T}_{i,j})} P(|Z_{ij \cdot \mathcal{K}} - z_{ij \cdot \mathcal{K}}| > \gamma) \leq O(N-n_N) \cdot \exp \Big( (N-n_N-2)\log \frac{4-(\gamma/L)^2}{4+(\gamma/L)^2} \Big),
\end{equation*}
for any $0 < \gamma \leq 2$, $L = \frac{1}{1-M^2}$. In the above formula, $Z_{ij \cdot \mathcal{K}} \triangleq F_z(r_{ij \cdot \mathcal{K}})$ and $z_{ij \cdot \mathcal{K}} \triangleq  F_z(\rho_{ij \cdot \mathcal{K}})$.
\end{lemma}

\begin{proof}
Based on the $z$-transform formula $F_z(\rho) = \frac{1}{2} \log (\frac{1+\rho}{1-\rho})$, the function $F_z(r)$ can be approximated at point $r=\rho$ using the first order Taylor series
\begin{align*}
     F_z(r) & \approx F_z(\rho) + F_z(\rho)'(r-\rho) \\
     & = F_z(\rho) + \frac{1}{1-\rho^2}(r-\rho).
\end{align*}
Next, we introduce the following basic fact from probability theory, if $|X| \geq 1$,
\begin{align*}
    P(|XY| > \gamma) = & P(|XY| > \gamma, |X| > L) + P(|XY| > \gamma, 1 \leq |X| \leq L) \\
    \leq & P(|X|> L) + P(|Y| > \gamma/L).
\end{align*}
Therefore, 
\begin{align*}
    P(|Z_{ij \cdot \mathcal{K}} - z_{ij \cdot \mathcal{K}}| > \gamma) = & P(|F_z(r_{ij \cdot \mathcal{K}}) - F_z(\rho_{ij \cdot \mathcal{K}})| > \gamma) \\
    = & P \Big( |\frac{1}{1-\rho_{ij \cdot \mathcal{K}}^2}(r_{ij \cdot \mathcal{K}}-\rho_{ij \cdot \mathcal{K}})| > \gamma \Big) \\
    \leq & P \Big( |\frac{1}{1-\rho_{ij \cdot \mathcal{K}}^2}| > L \Big) + P \Big( |r_{ij \cdot \mathcal{K}}-\rho_{ij \cdot \mathcal{K}})| > \frac{\gamma}{L} \Big),
\end{align*}
where $L = \frac{1}{1-M^2}$.

The condition (C3) indicates $\rho_{ij \cdot \mathcal{K}} \leq M$, thus $P \big( |\frac{1}{1-\rho_{ij \cdot \mathcal{K}}^2}| \leq L = \frac{1}{1-M^2} \big) = 1$, and the first term of the above equation equals to zero. Based on Lemma \ref{lemma:partialCorrelationBound}, the second term is upper-bounded by $C_1(M) \cdot (N-n_N) \cdot \exp \Big( (N-n_N-2)\log \frac{4-(\frac{\gamma}{L})^2}{4+(\frac{\gamma}{L})^2} \Big)$, where $C_1(M)$ is a function defined in Proposition \ref{prop:prop1} and is only dependent on $M$. The proof is complete.
\end{proof}

\textbf{Proof of Theorem \ref{them:skeletonconsistency}} Next, we use the Lemma \ref{lemma:correlationBound}, \ref{lemma:partialCorrelationBound}, \ref{lemma:zBound} to prove the Theorem \ref{them:skeletonconsistency}. If an error occurs when testing edge $X_i - X_j$ conditioned on set $\mathcal{K}$ in the proposed algorithm, it must be the following two cases: (i) the edge $X_i - X_j$ does not exist in the true DAG but is kept by our algorithm and (ii) the edge $X_i - X_j$ exists in the DAG but is deleted by our algorithm. We denote the former event as a Type I error and the later as a Type II error, and use $E_{ij \cdot \mathcal{K}}^I$ and $E_{ij \cdot \mathcal{K}}^{II}$ to represent them, respectively.
\begin{align*}
    \text{Event } E_{ij \cdot \mathcal{K}}^I: |Z_{ij \cdot \mathcal{K}}| > \frac{\Phi^{-1}(1-\frac{\alpha}{2})}{\sqrt{N - |\mathcal{K}| - 3}} \text{ but } z_{ij \cdot \mathcal{K}} = 0, \\
    \text{Event } E_{ij \cdot \mathcal{K}}^{II}: |Z_{ij \cdot \mathcal{K}}| \leq \frac{\Phi^{-1}(1-\frac{\alpha}{2})}{\sqrt{N - |\mathcal{K}| - 3}} \text{ but } z_{ij \cdot \mathcal{K}} \neq 0.
\end{align*}
Therefore,
\begin{align*}
    & P(\text{Error happens in the reverse order pruning PC algorithm}) \\
    = & P \Bigg( \bigcup_{(i, j, \mathcal{T}_{i,j}^m)}E_{ij \cdot \mathcal{K}}^I  \cup \bigcup_{(i, j, \mathcal{T}_{i,j})}E_{ij \cdot \mathcal{K}}^{II} \Bigg) \\
    = & \bigcup_{(i, j, \mathcal{T}_{i,j}^m)} P(E_{ij \cdot \mathcal{K}}^I)  \cup \bigcup_{(i, j, \mathcal{T}_{i,j})} P(E_{ij \cdot \mathcal{K}}^{II}) \\
    \leq & O(|(i, j, \mathcal{T}_{i,j}^m)|) \sup_{(i, j, \mathcal{T}_{i,j}^m)} P(E_{ij \cdot \mathcal{K}}^I) + O(|(i, j, \mathcal{T}_{i,j})|) \sup_{(i, j, \mathcal{T}_{i,j})} P(E_{ij \cdot \mathcal{K}}^{II}).
\end{align*}

Based on the definition of $\mathcal{T}_{i,j}^m$ and $\mathcal{T}_{i,j}$,
\begin{equation*}
    O(|(i, j, \mathcal{T}_{i,j}^m)|) = O \Bigg( \binom{n_N}{2}\sum_{l=n_N-2-v}^{n_N-2}\binom{n_N-2}{l} \Bigg) = O(n_N^{v+2}),
\end{equation*}
and 
\begin{equation*}
    O(|(i, j, \mathcal{T}_{i,j})|) = O \Bigg( \binom{n_N}{2}\sum_{l=0}^{n_N-2}\binom{n_N-2}{l} \Bigg) = O(2^{n_N-2}).
\end{equation*}

By choosing $\alpha = 2(1-\Phi(\sqrt{N}c_N/2))$ using the $c_N$ defined in (C3), we have
\begin{align*}
    \sup_{(i, j, \mathcal{T}_{i,j}^m)} P(E_{ij \cdot \mathcal{K}}^I) = & \sup_{(i, j, \mathcal{T}_{i,j}^m)} P\Bigg(|Z_{ij \cdot \mathcal{K}} - z_{ij \cdot \mathcal{K}}| > \frac{\sqrt{N}c_N/2}{\sqrt{N-|\mathcal{K}| - 3}}\Bigg) \\
    \leq & O(N-n_N) \cdot \exp \Bigg( (N-n_N-2) \log \frac{4-\frac{Nc_N^2}{4(N-|\mathcal{K}|-3)L^2}}{4+\frac{Nc_N^2}{4(N-|\mathcal{K}|-3)L^2}}\Bigg).
\end{align*}
The last inequality is from Lemma \ref{lemma:zBound}. Since the $c_N$ has an order of $O(N^{-b})$ from (C3), the $\log$ term in the end can be approximated using the following fact
\begin{align*}
    \log \frac{4-x}{4+x} \approx \frac{4+x}{4-x}\frac{-(4+x)-(4-x)}{(4+x)^2}=\frac{-8x}{16-x^2} \rightarrow -\frac{x}{2} \text{ as } x \rightarrow 0.
\end{align*}
Using $x=\frac{Nc_N^2}{4(N-|\mathcal{K}|-3)L^2}$, we have 
\begin{align*}
    \sup_{(i, j, \mathcal{T}_{i,j}^m)} P(E_{ij \cdot \mathcal{K}}^I) \leq O(N-n_N) \cdot \exp(-C_2(M) \cdot (N-n_N-2)c_N^2),
\end{align*}
where $C_2(M) = \frac{D_1}{L^2}$ where $D_1 = \frac{N}{8(N-|\mathcal{K}|-3)}$ is a constant in $(\frac{1}{8}, 1)$. Thus $C_2(M)$ is a function of only $M$ in the range $(\frac{1}{8}, 1)$ when $M$ is in $(0,1)$.

Similarly, we have 
\begin{align*}
    \sup_{(i, j, \mathcal{T}_{i,j})} P( E_{ij \cdot \mathcal{K}}^{II}) = & \sup_{(i, j, \mathcal{T}_{i,j})} P\Bigg(|Z_{ij \cdot \mathcal{K}}| \leq \frac{\Phi^{-1}(1-\frac{\alpha}{2})}{\sqrt{N - |\mathcal{K}| - 3}}\Bigg) \\
    \leq & \sup_{(i, j, \mathcal{T}_{i,j})} P\Bigg(|z_{ij \cdot \mathcal{K}}| - |Z_{ij \cdot \mathcal{K}}| > c_N \Big( 1-\frac{\sqrt{N}}{2\sqrt{N-|\mathcal{K}| - 3}} \Big) \Bigg),\\
    \leq & \sup_{(i, j, \mathcal{T}_{i,j})} P\Bigg(|Z_{ij \cdot \mathcal{K}} - z_{ij \cdot \mathcal{K}}| > c_N \Big( 1-\frac{\sqrt{N}}{2\sqrt{N-|\mathcal{K}| - 3}} \Big) \Bigg),
\end{align*}
where the second inequality is based on the fact that $\sup_{(i, j, \mathcal{T}_{i,j}^m)} |z_{ij \cdot \mathcal{K}}| \geq \sup_{(i, j, \mathcal{T}_{i,j}^m)} |\rho_{ij \cdot \mathcal{K}}| \geq c_N$, and the last inequality is because $|a-b| \geq ||a| - |b||$. 

Denote $D_2 = 1-\frac{\sqrt{N}}{2\sqrt{N-|\mathcal{K}| - 3}}$ which is a constant in range $(0, 1/2)$. By Lemma \ref{lemma:zBound}, we have 
\begin{align*}
    \sup_{(i, j, \mathcal{T}_{i,j})} P(E_{ij \cdot \mathcal{K}}^{II}) \leq O(N-n_N) \cdot \exp(-C_3(M) \cdot (N-n_N-2)c_N^2),
\end{align*}
where $C_3(M) = \frac{D_2^2}{2L^2} = \frac{D_2^2(1-M^2)^2}{2}$ which is a function of only $M$, and is in the range of $(0, \frac{1}{8})$.

In the end, we have 
\begin{align*}
    & P(\text{Error happens in the reverse order pruning PC algorithm}) \\
    \leq & O((n_N^{v+2} + 2^{n_N-2})(N-n_N)\exp(-(N-n_N-2)c_N^2)) \\
    = & O((n_N^{v+2} + 2^{n_N-2})(N-n_N)\exp(-N^{1-2b})\exp(n_N+2))  \\
    \stackrel{\text{(a)}}{=} & O(2^{n_N-2}N\exp(-N^{1-2b})\exp(n_N+2))  \\
    \stackrel{\text{(b)}}{\leq} & O(\exp(-N^{1-2b})) \rightarrow 0 \text{ as } N \rightarrow \infty,
\end{align*}
where $(a)$ is because
\begin{align*}
    & \lim_{N \rightarrow \infty}\frac{n_N^{v+2}}{2^{n_N-2}} = \lim_{N \rightarrow \infty}\frac{N^{d(N^e+2)}}{2^{N^d-2}} = \lim_{N \rightarrow \infty} 2^{d(N^e+2)\log_2N - (N^d-2)} \\
    = & 2^{\lim_{N \rightarrow \infty} dN^e\log_2N - N^d + 2d\log_2(N) +2} \rightarrow 2^{-\infty} = 0,
\end{align*}
because of $e < d$ in condition (C4), $-N^d$ is the dominant term, moreover, $N$ dominates $n_N$ since $n_N = O(N^d)$ in condition (C2); inequality (b) is because 
\begin{align*}
    & \lim_{N \rightarrow \infty} \frac{2^{n_N-2}N e^{n_N+2}}{e^{N^{1-2b}}} = \lim_{N \rightarrow \infty} \frac{2^{N^d-2}N e^{N^d+2}}{e^{N^{1-2b}}} = \lim_{N \rightarrow \infty} \frac{ e^{(N^d-2)\ln2+ \ln N +N^d+2}}{e^{N^{1-2b}}} \\
    = & e^{\lim_{N \rightarrow \infty} (N^d-2)\ln2 + \ln N + N^d + 2 - N^{1-2b}} \rightarrow e^{-\infty} = 0,
\end{align*}
because of $b < (1-d)/2$ in condition (C3), $-N^{1-2d}$ is the dominant term.

\subsection{Proof of Theorem \ref{them:cpdagconsistency}}\label{appendix:them:cpdagconsistency}
\begin{proof}
The Phase II of the algorithm infers directions of edges using matrix $S$. By Proposition \ref{prop:prop2}, if the $\mathcal{K}'$ in each entry of matrix $S$ satisfy Proposition \ref{prop:prop1}, the CPDAG will be the true CPDAG by the oracle version of the algorithm. The correctness of each $\mathcal{K}'$ is ensured by the correctness of Phase I of the algorithm. Therefore,
\begin{align*}
    & P(\widehat{G}_{CPDAG} = G_{CPDAG}) \\
    = & P(\hat{G}_{skeleton} = G_{skeleton}) P(\text{Phase II is correct } \vert \; \hat{G}_{skeleton} = G_{skeleton}) \\
    = & P(\hat{G}_{skeleton} = G_{skeleton}) \rightarrow 1 \text{ as } N \rightarrow \infty.
\end{align*}
\end{proof}

\vskip 0.2in

\bibliographystyle{IEEEtran}
\bibliography{myReferences}

\begin{thebibliography}{10}
\providecommand{\url}[1]{#1}
\csname url@samestyle\endcsname
\providecommand{\newblock}{\relax}
\providecommand{\bibinfo}[2]{#2}
\providecommand{\BIBentrySTDinterwordspacing}{\spaceskip=0pt\relax}
\providecommand{\BIBentryALTinterwordstretchfactor}{4}
\providecommand{\BIBentryALTinterwordspacing}{\spaceskip=\fontdimen2\font plus
\BIBentryALTinterwordstretchfactor\fontdimen3\font minus
  \fontdimen4\font\relax}
\providecommand{\BIBforeignlanguage}[2]{{%
\expandafter\ifx\csname l@#1\endcsname\relax
\typeout{** WARNING: IEEEtran.bst: No hyphenation pattern has been}%
\typeout{** loaded for the language `#1'. Using the pattern for}%
\typeout{** the default language instead.}%
\else
\language=\csname l@#1\endcsname
\fi
#2}}
\providecommand{\BIBdecl}{\relax}
\BIBdecl

\bibitem{pearl2009causal}
J.~Pearl \emph{et~al.}, ``Causal inference in statistics: An overview,''
  \emph{Statistics surveys}, vol.~3, pp. 96--146, 2009.

\bibitem{zajonc2012essays}
T.~Zajonc, ``Essays on causal inference for public policy,'' Ph.D.
  dissertation, Harvard University, 2012.

\bibitem{varian2016causal}
H.~R. Varian, ``Causal inference in economics and marketing,''
  \emph{Proceedings of the National Academy of Sciences}, vol. 113, no.~27, pp.
  7310--7315, 2016.

\bibitem{shipley2016cause}
B.~Shipley, \emph{Cause and correlation in biology: a user's guide to path
  analysis, structural equations and causal inference with R}.\hskip 1em plus
  0.5em minus 0.4em\relax Cambridge University Press, 2016.

\bibitem{glass2013causal}
T.~A. Glass, S.~N. Goodman, M.~A. Hern{\'a}n, and J.~M. Samet, ``Causal
  inference in public health,'' \emph{Annual review of public health}, vol.~34,
  pp. 61--75, 2013.

\bibitem{spirtes1991algorithm}
P.~Spirtes and C.~Glymour, ``An algorithm for fast recovery of sparse causal
  graphs,'' \emph{Social science computer review}, vol.~9, no.~1, pp. 62--72,
  1991.

\bibitem{singh2017comparative}
K.~Singh, G.~Gupta, V.~Tewari, and G.~Shroff, ``Comparative benchmarking of
  causal discovery techniques,'' \emph{arXiv preprint arXiv:1708.06246}, 2017.

\bibitem{tong2001active}
S.~Tong and D.~Koller, ``Active learning for structure in bayesian networks,''
  in \emph{International joint conference on artificial intelligence},
  vol.~17.\hskip 1em plus 0.5em minus 0.4em\relax Citeseer, 2001, pp. 863--869.

\bibitem{scheines1998tetrad}
R.~Scheines, P.~Spirtes, C.~Glymour, C.~Meek, and T.~Richardson, ``The tetrad
  project: Constraint based aids to causal model specification,''
  \emph{Multivariate Behavioral Research}, vol.~33, no.~1, pp. 65--117, 1998.

\bibitem{kalainathan2019causal}
D.~Kalainathan and O.~Goudet, ``Causal discovery toolbox: Uncover causal
  relationships in python,'' \emph{arXiv preprint arXiv:1903.02278}, 2019.

\bibitem{kalisch2012causal}
M.~Kalisch, M.~M{\"a}chler, D.~Colombo, M.~H. Maathuis, and P.~B{\"u}hlmann,
  ``Causal inference using graphical models with the r package pcalg,''
  \emph{Journal of Statistical Software}, vol.~47, no.~11, pp. 1--26, 2012.

\bibitem{scutari2009learning}
M.~Scutari, ``Learning bayesian networks with the bnlearn r package,''
  \emph{arXiv preprint arXiv:0908.3817}, 2009.

\bibitem{le2016fast}
T.~Le, T.~Hoang, J.~Li, L.~Liu, H.~Liu, and S.~Hu, ``A fast pc algorithm for
  high dimensional causal discovery with multi-core pcs,'' \emph{IEEE/ACM
  transactions on computational biology and bioinformatics}, 2016.

\bibitem{madsen2017parallel}
A.~L. Madsen, F.~Jensen, A.~Salmer{\'o}n, H.~Langseth, and T.~D. Nielsen, ``A
  parallel algorithm for bayesian network structure learning from large data
  sets,'' \emph{Knowledge-Based Systems}, vol. 117, pp. 46--55, 2017.

\bibitem{zare2018cupc}
B.~Zare, F.~Jafarinejad, M.~Hashemi, and S.~Salehkaleybar, ``cupc: Cuda-based
  parallel pc algorithm for causal structure learning on gpu,'' \emph{arXiv
  preprint arXiv:1812.08491}, 2018.

\bibitem{heckerman1995learning}
D.~Heckerman, D.~Geiger, and D.~M. Chickering, ``Learning bayesian networks:
  The combination of knowledge and statistical data,'' \emph{Machine learning},
  vol.~20, no.~3, pp. 197--243, 1995.

\bibitem{lauritzen1996graphical}
S.~L. Lauritzen, \emph{Graphical models}.\hskip 1em plus 0.5em minus
  0.4em\relax Clarendon Press, 1996, vol.~17.

\bibitem{spirtes2000causation}
P.~Spirtes, C.~N. Glymour, R.~Scheines, and D.~Heckerman, \emph{Causation,
  prediction, and search}.\hskip 1em plus 0.5em minus 0.4em\relax MIT press,
  2000.

\bibitem{hausman1999independence}
D.~M. Hausman and J.~Woodward, ``Independence, invariance and the causal markov
  condition,'' \emph{The British journal for the philosophy of science},
  vol.~50, no.~4, pp. 521--583, 1999.

\bibitem{judea2000causality}
P.~Judea, ``Causality: models, reasoning, and inference,'' \emph{Cambridge
  University Press. ISBN 0}, vol. 521, no. 77362, p.~8, 2000.

\bibitem{cartwright1984laws}
N.~Cartwright and E.~McMullin, ``How the laws of physics lie,'' 1984.

\bibitem{chickering2002optimal}
D.~M. Chickering, ``Optimal structure identification with greedy search,''
  \emph{Journal of machine learning research}, vol.~3, no. Nov, pp. 507--554,
  2002.

\bibitem{meek2013causal}
C.~Meek, ``Causal inference and causal explanation with background knowledge,''
  \emph{arXiv preprint arXiv:1302.4972}, 2013.

\bibitem{verma1991equivalence}
T.~Verma and J.~Pearl, \emph{Equivalence and synthesis of causal models}.\hskip
  1em plus 0.5em minus 0.4em\relax UCLA, Computer Science Department, 1991.

\bibitem{he2015counting}
Y.~He, J.~Jia, and B.~Yu, ``Counting and exploring sizes of markov equivalence
  classes of directed acyclic graphs,'' \emph{The Journal of Machine Learning
  Research}, vol.~16, no.~1, pp. 2589--2609, 2015.

\bibitem{colombo2014order}
D.~Colombo and M.~H. Maathuis, ``Order-independent constraint-based causal
  structure learning,'' \emph{The Journal of Machine Learning Research},
  vol.~15, no.~1, pp. 3741--3782, 2014.

\bibitem{glymour2019review}
C.~Glymour, K.~Zhang, and P.~Spirtes, ``Review of causal discovery methods
  based on graphical models,'' \emph{Frontiers in genetics}, vol.~10, p. 524,
  2019.

\bibitem{fisher1992statistical}
R.~A. Fisher, ``Statistical methods for research workers,'' in
  \emph{Breakthroughs in statistics}.\hskip 1em plus 0.5em minus 0.4em\relax
  Springer, 1992, pp. 66--70.

\bibitem{fisher1924distribution}
------, ``The distribution of the partial correlation coefficient,''
  \emph{Metron}, vol.~3, pp. 329--332, 1924.

\bibitem{colombo2012modification}
D.~Colombo and M.~H. Maathuis, ``A modification of the pc algorithm yielding
  order-independent skeletons,'' \emph{arXiv preprint arXiv:1211.3295}, 2012.

\bibitem{kalisch2007estimating}
M.~Kalisch and P.~B{\"u}hlmann, ``Estimating high-dimensional directed acyclic
  graphs with the pc-algorithm,'' \emph{Journal of Machine Learning Research},
  vol.~8, no. Mar, pp. 613--636, 2007.

\bibitem{hotelling1953new}
H.~Hotelling, ``New light on the correlation coefficient and its transforms,''
  \emph{Journal of the Royal Statistical Society. Series B (Methodological)},
  vol.~15, no.~2, pp. 193--232, 1953.

\end{thebibliography}

\end{document}